\documentclass{article}
\usepackage[utf8]{inputenc}

\usepackage{arxiv}
\usepackage{times}

\usepackage{amsmath}
\DeclareMathOperator*{\argmin}{arg\,min}
\usepackage{amssymb}
\usepackage{amsthm}
\usepackage{mathtools}
\usepackage{booktabs}
\usepackage{enumitem}
\usepackage{xcolor}
\usepackage{caption}
\usepackage{placeins}
\usepackage[numbers]{natbib}
\usepackage{hyperref}

\usepackage{bm}

\newcommand{\R}{\mathbb{R}}

\newcommand{\cN}{\mathrm{N}}  
\newcommand{\cK}{\mathrm{K}}  
\newcommand{\cd}{\mathrm{d}}  
\newcommand{\ch}{\mathrm{h}}  
\newcommand{\cL}{\mathrm{L}}  
\newcommand{\cD}{\mathrm{D}}  

\newcommand{\sV}{\mathcal{V}}
\newcommand{\sE}{\mathcal{E}}
\newcommand{\sC}{\mathcal{C}}
\newcommand{\sT}{\mathcal{T}}  
\newcommand{\sB}{\mathcal{B}}

\newcommand{\cF}{\mathcal{F}}

\newcommand{\vx}{\bm{x}}
\newcommand{\vy}{\bm{y}}
\newcommand{\vz}{\bm{z}}
\newcommand{\vh}{\bm{h}}
\newcommand{\vd}{\bm{d}}
\newcommand{\vm}{\bm{m}}
\newcommand{\vphi}{\bm{\phi}}
\newcommand{\vone}{\bm{1}}
\newcommand{\vzero}{\bm{0}}

\newcommand{\mA}{\bm{A}}   
\newcommand{\mD}{\bm{D}}   
\newcommand{\mH}{\bm{H}}   
\newcommand{\mI}{\bm{I}}   
\newcommand{\mL}{\bm{L}}   
\newcommand{\mS}{\bm{S}}   
\newcommand{\mU}{\bm{U}}   
\newcommand{\mW}{\bm{W}}   
\newcommand{\mX}{\bm{X}}   
\newcommand{\mZ}{\bm{Z}}   
\newcommand{\mK}{\bm{K}}   

\newcommand{\ermS}{S}

\usepackage{url}
\usepackage{wrapfig}
\usepackage{graphicx}
\usepackage{listings}

\newcommand{\rebuttal}[1]{#1}
\newenvironment{rebuttalblock}{}{}

\definecolor{hl-blue}{RGB}{0, 0, 180}
\definecolor{hl-green}{RGB}{0, 128, 0}
\definecolor{hl-red}{RGB}{196, 0, 0}
\hypersetup{
    colorlinks=true,       
    linkcolor=hl-red,      
    citecolor=hl-blue,     
    filecolor=magenta,     
    urlcolor=hl-green         
}

\providecommand{\cref}[1]{\ref{#1}}
\providecommand{\Cref}[1]{\ref{#1}}
\usepackage[acronym, nohypertypes={acronym}]{glossaries}
\glsdisablehyper

\newacronym{tgp}{\texttt{tgp}}{Torch Geometric Pool}
\newacronym{pyg}{\texttt{PyG}}{Pytorch Geometric}
\newacronym{dgl}{DGL}{Deep Graph Library}
\newacronym{gsp}{GSP}{Graph Signal Processing}
\newacronym{src}{SRC}{Select--Reduce--Connect}
\newacronym{srcl}{SRCL}{Select-Reduce-Connect-Lift}
\newacronym{nmi}{NMI}{Normalized Mutual Information}
\newacronym{bfs}{BFS}{Breadth-First Search}

\newacronym{cnn}{CNN}{Convolutional Neural Network}
\newacronym{gnn}{GNN}{Graph Neural Network}
\newacronym{mlp}{MLP}{Multilayer Perceptron}
\newacronym{gcn}{GCN}{Graph Convolutional Network}
\newacronym{gin}{GIN}{Graph Isomorphism Network}
\newacronym{snn}{SNN}{Sheaf Neural Network}
\newacronym{hsnn}{HSNN}{Hierarchical Sheaf Neural Network}

\newacronym{hisp}{HiSP}{Hierarchical Sheaf Pooling}
\newacronym{nsd}{NSD}{Neural Sheaf Diffusion}
\newacronym{mp}{MP}{Message Passing}
\newacronym{ogcn}{$\omega$-GCN}{$\omega$-Graph Convolution Network}
\newacronym{hetmp}{HetMP}{Heterophilic Message Passing}

\newacronym{hosc}{HOSC}{High-Order Spectral Clustering}
\newacronym{asap}{ASAP}{Adaptive Structure Aware Pooling}
\newacronym{jbgnn}{JBPool}{Just-Balance Pooling}
\newacronym{ecpool}{ECPool}{Edge-Contraction Pooling}
\newacronym{ndp}{NDP}{Node Decimation Pooling}
\newacronym{nmf}{NMF}{Non-negative Matrix Factorization pooling}
\newglossaryentry{topk}{name=Top-$k$,description=}
\newacronym{dmon}{DMoN}{Deep Modularity Networks}
\newacronym{acc}{ACC}{Asymmetric Cheeger Cut pooling}
\newacronym{kmis}{$\kappa$-MIS}{$\kappa$-Maximal Independent Sets pooling}
\newglossaryentry{mincut}{name=MinCut,description=}
\newglossaryentry{maxcut}{name=MaxCut,description=}
\newacronym{tvgnn}{TVGNN}{Total Variation Graph Neural Network}
\newacronym{pan}{PAN}{Path integral based pooling}
\newacronym{sag}{SAG}{Self-Attention Graph pooling}
\newacronym{bnpool}{BNPool}{Bayesian Nonparametric Pooling}
\newglossaryentry{graclus}{name=GraClus,description=}
\newacronym{lapool}{LaPool}{Laplacian Pooling}
\newacronym{sep}{SEP}{Structural Entropy Pooling}
\newglossaryentry{eigen}{name=EigenPool,description=}
\newglossaryentry{diff}{name=DiffPool,description=}
\newglossaryentry{nopool}{name=NoPool,description=}

\newacronym{ogb}{ogbg-ppa}{protein-protein association networks}
\newacronym{gcb}{GCB-H}{Graph Classification Benchmark Hard}
\newacronym{lrgb}{LRGB}{Long-Range Graph Benchmark}
\newacronym{hierbi}{HierBI}{Hierarchical BIpartite}
\newacronym{lrim}{LRIM}{Long-Range Ising Model}

\let\origgls\gls
\let\origacrshort\acrshort
\let\origacrlong\acrlong
\let\origacrfull\acrfull
\renewcommand{\gls}[1]{\ifglsentryexists{#1}{\origgls{#1}}{\texttt{#1}}}
\renewcommand{\acrshort}[1]{\ifglsentryexists{#1}{\origacrshort{#1}}{\texttt{#1}}}
\renewcommand{\acrlong}[1]{\ifglsentryexists{#1}{\origacrlong{#1}}{\texttt{#1}}}
\renewcommand{\acrfull}[1]{\ifglsentryexists{#1}{\origacrfull{#1}}{\texttt{#1}}}

\title{Hierarchical Pooling for Sheaf Neural Networks}

\author{%
  Dionisia Naddeo\thanks{Equal contribution.}  \\
  Department of Biomedicine and Prevention\\ University of Rome, Tor Vergata 
  \And
  Carlo Abate\textsuperscript{*} \\
  UiT The Arctic University of Norway
  \And
  Pietro Liò \\
  University of Cambridge
  \And
  Nicola Toschi \\
  Department of Biomedicine and Prevention\\ University of Rome, Tor Vergata\\
    Martinos Center For Biomedical Imaging\\
    MGH and Harvard Medical School (USA)
  \And
  Filippo Maria Bianchi \\
  UiT The Arctic University of Norway\\
  NORCE Norwegian Research Centre AS
}

\begin{document}

\maketitle

\begin{abstract}
\glspl{snn} extend \glspl{gnn} by operating on cellular sheaves, which assign to each node and edge its own vector space, along with linear restriction maps from node spaces to the spaces associated with their incident edges. These maps determine the sheaf Laplacian and can be learned at each layer, enabling layer-specific diffusion dynamics. 
Existing \gls{snn} architectures, however, always diffuse on the same graph: restriction maps may change across layers, but they continue to act on the original nodes and edges. We introduce \gls{hisp}, a pooling layer that enables hierarchical representation learning in \glspl{snn}. Following the \gls{srcl} framework, \gls{hisp} groups fine-scale nodes, reduces their features using the fine-scale sheaf structure, and constructs the corresponding coarse graph. A new sheaf is then learned at the coarse resolution. By interleaving sheaf diffusion and pooling, \gls{hisp} enables hierarchical \glspl{snn} to learn distinct diffusion dynamics across layers and graph resolutions.  This construction supports multiscale problems in which aggregation changes both the interacting entities and their relations. 
\end{abstract}

\section{Introduction}
\label{sec:introduction}


\glspl{gnn} update node representations by repeatedly aggregating information from neighboring nodes \citep{gori2005new,scarselli2008graph}. Many widely used \gls{gnn} layers can be interpreted as diffusion or filtering with respect to the graph Laplacian, which promotes agreement between neighboring representations that already live in a shared feature space \citep{defferrard2016convolutional,kipf2016semi,bianchi2021graph}.
\glspl{snn} generalize this diffusion view using cellular sheaves, which assign a vector space (called a stalk) to each vertex and edge, and functions (called restriction maps) that send vertex data from their node stalks to the stalks of their incident edges~\citep{hansen2020sheaf,bodnar2022neural,barbero2022sheafatt,barbero2022sheaf,zaghen2024sheaf,braithwaite2024heterogeneous,purificato2025sheaf4rec,borgi2025polynomial}.
Unlike \glspl{gnn}, neighboring features need not share coordinates: restriction maps first send them into a common edge stalk, and the sheaf Laplacian promotes agreement only after that alignment.
When the restriction maps are built from latent node representations, different layers can learn different notions of compatibility and, consequently, different diffusion dynamics \citep{bodnar2022neural}.

Existing \gls{snn} architectures always diffuse on the same graph: restriction maps may change across layers, but continue to describe interactions between nodes of the original graph \citep{bodnar2022neural}. Pooling can create a hierarchy of entities, allowing subsequent layers to model interactions specific to each resolution. This is essential to capture relations and interaction dynamics that change across scales. In opinion dynamics, for example, fine-scale restriction maps describe how individuals express private opinions in interpersonal discourse spaces \citep{hansen2021opinion}. A pooled node may represent the aggregate position of a social group or a political party. Relations between these coarse entities may encode institutional mechanisms such as coalition, negotiation, or opposition, which differ from interpersonal communication. 

Hierarchical \glspl{gnn} address multiscale structure by repeatedly dropping nodes or grouping them into supernodes, thereby constructing local summaries at progressively lower spatial resolutions \citep{grattarola2022understanding}. This process quickly expands the receptive field of a \gls{gnn}, reduces the computational graph, and progressively distills local information into global representations. Several pooling operators cluster nodes and aggregate their features to form supernodes \citep{ying2018hierarchical,bianchi2019mincut,grattarola2022understanding,abate2025maxcutpool}. 
In an \gls{snn}, however, each node representation belongs to a distinct vertex stalk \citep{hansen2019toward}. Even when all stalks have the same dimension, their coordinates might have different semantics, so a pooling operator cannot aggregate them as if they lived in a shared Euclidean feature space. \rebuttal{When the dimensions differ, the features cannot be added at all.}
A sheaf-aware pooling layer must additionally define how representations from different stalks are related, aggregated, and represented on the coarse graph. 

We introduce \emph{\acrfull{hisp}}, the first sheaf-aware hierarchical pooling operator that progressively coarsens a sheaf by clustering nodes defined in different vertex stalks. Within each cluster, \gls{hisp} uses the learned restriction maps to construct \emph{transport maps} that project node representations into a common reference stalk. The aligned representations are then aggregated into a single coarse-node representation. The resulting coarse topology is passed to a subsequent sheaf layer, which learns new restriction maps between coarse nodes. The sheaf structure and its associated diffusion dynamics can thus adapt to the resolution produced by pooling.

We further provide the components required for end-to-end graph-level learning with \glspl{snn}. Specifically, we implement neural sheaf diffusion \citep{bodnar2022neural} as a message-passing layer in \gls{pyg}~\citep{fey2019fast}, cast \gls{hisp} into the modular \acrfull{srcl} framework~\citep{grattarola2022understanding}, and implement it into the \gls{tgp} workflow \citep{abate2026torchgeometricpoolpytorch}. 
We also introduce a sheaf-aware graph readout. Together with \gls{hisp}, these components provide a complete \gls{pyg} workflow for hierarchical \glspl{snn} on graph-level tasks.
The code is available at \url{https://github.com/NGMLGroup/hierarchical-pooling-for-sheaf-neural-networks}.


\section{Background and Related Work}
\label{sec:background}

\subsection{Cellular sheaves and sheaf Laplacians}
\label{subsec:cellular_sheaves_and_sheaf_laplacian}
Let \(G=(\sV,\sE)\) be a graph with \(\cN = |\sV|\) vertices. A \emph{cellular sheaf} \(\cF\) on \(G\) assigns a finite-dimensional inner-product space \(\cF(v)\) to each vertex \(v\), and \(\cF(e)\) to each edge \(e\), together with a linear \emph{restriction map} \(\rho_{e,v}:\cF(v)\longrightarrow\cF(e)\) for every incident pair \(v\trianglelefteq e\) \citep{hansen2019toward, hansen2021opinion}. 
The spaces \(\cF(v)\) and \(\cF(e)\) are called \emph{stalks}. A \emph{$0$-cochain} is a choice of a vector \(\vx_v \in \cF(v)\) for each \(v \in \sV\); a \emph{$1$-cochain} is a choice of \(\vx_e \in \cF(e)\) for each \(e\in\sE\). The spaces of $0$- and $1$-cochains are the direct sums of the stalks over vertices and edges,
\( C^0(G;\cF) = \bigoplus_{v\in \sV}\cF(v)\) and \(C^1(G;\cF)= \bigoplus_{e\in \sE}\cF(e)\),
equipped with the corresponding direct-sum inner products. After fixing an arbitrary edge orientation, the coboundary operator \(\delta_{\cF}:C^0(G;\cF)\rightarrow C^1(G;\cF)\) measures the disagreement between the endpoint restrictions and is defined by \((\delta_{\cF}\vx)_e=\rho_{e,u}\vx_u-\rho_{e,v}\vx_v\) for every oriented edge \(e=(u,v)\).
The associated \emph{combinatorial sheaf Laplacian} is
\begin{equation}
\label{eq:sheaf_laplacian}
\mL_{\cF} :=\delta_{\cF}^{\ast}\delta_{\cF}: C^0(G;\cF)\longrightarrow C^0(G;\cF),
\end{equation}
where \(\delta_{\cF}^{*}\) denotes the adjoint with respect to the direct-sum inner products. \(\mL_{\cF}\) is self-adjoint, positive semidefinite, and independent of the chosen edge orientation. Its
quadratic form
\begin{equation}
\label{eq:sheaf_dirichlet_energy}
\left\langle \vx,\mL_{\cF}\vx\right\rangle = \left\|\delta_{\cF}\vx\right\|^2 = \sum_{e=(u,v)\in \sE} \left\| \rho_{e,u}\vx_u - \rho_{e,v}\vx_v \right\|^2,
\end{equation}
measures the total disagreement between adjacent node representations in their common edge stalks.

\subsection{Sheaf Neural Networks}
\label{subsec:sheaf_neural_networks}
Let \(\cD_0=\dim C^0(G;\cF)\). \rebuttal{For notational simplicity we assume that all vertex and edge stalks are isomorphic to \(\R^{\cd}\), so that \(\cD_0=\cN\,\cd\). The constructions below do not use that assumption. When \(\dim\cF(u)\neq\dim\cF(v)\), the two features cannot be added, and the transport of Section~\ref{subsec:reduce} is the operator that remains defined. Equal dimension is the special case in which a coordinate-wise mean is typed, and it still compares coordinates that need not share a basis.} An \(\ch\)-channel sheaf signal is represented by a matrix \(\mX\in\R^{(\cN\,\cd)\times\ch}\), whose columns are \(0\)-cochains. 
\glspl{snn} propagate these signals using a diffusion operator induced by the sheaf Laplacian, whose effect is to reduce the disagreement between neighboring node representations after they have been mapped to their shared edge stalks. A sheaf convolution augments this process with learnable transformations of the stalk coordinates and feature channels, followed by a nonlinear activation \citep{hansen2020sheaf}. \gls{nsd} extends this construction by learning a layer-dependent sheaf from the current latent representations  \citep{bodnar2022neural}. 
At layer \(\ell\), let \(\mX_u^{(\ell)}\in\R^{\cd\times\ch}\) denote the slice of \(\mX^{(\ell)}\) associated with vertex \(u\). For an edge \(e=\{u,v\}\), the endpoint restriction maps are parameterized as
\begin{equation}
\label{eq:sheaf_learner}
\rho_{e,u}^{(\ell)} = \Phi_{\theta_\ell}
\left(\mX_u^{(\ell)},\mX_v^{(\ell)}\right),
\qquad
\rho_{e,v}^{(\ell)} = \Phi_{\theta_\ell} \left(\mX_v^{(\ell)},\mX_u^{(\ell)}\right),
\end{equation}
where \(\Phi_{\theta_\ell}\) is a matrix-valued function implemented by an MLP \rebuttal{(or a GNN in node-level bundle and directed-sheaf variants~\citep{bamberger2025bundle,ribeiro2025cooperative})} with trainable parameters $\theta_\ell$ followed by reshaping. Its outputs may be parameterized as general, diagonal, or orthogonal matrices.
Thus, the restriction maps, and consequently the diffusion operator, may change across layers. Let \(\mD^{(\ell)}\) denote the block-diagonal part of \(\mL_{\cF^{(\ell)}}\). The normalized sheaf
Laplacian is \( \Delta_{\cF^{(\ell)}} := \bigl(\mD^{(\ell)}\bigr)^{-1/2} \mL_{\cF^{(\ell)}} \bigl(\mD^{(\ell)}\bigr)^{-1/2}.\) The node representations are updated as
\begin{equation}
\label{eq:neural_sheaf_diffusion}
\mX^{(\ell+1)} =
\mX^{(\ell)} - \sigma\!\left[
\Delta_{\cF^{(\ell)}}
\left(\mI_{\cN}\otimes\mW_1^{(\ell)}\right)
\mX^{(\ell)}
\mW_2^{(\ell)}
\right].
\end{equation}
Here, \(\sigma\) denotes a nonlinear activation, \(\mW_1^{(\ell)}\in\R^{\cd\times\cd}\) mixes coordinates within each vertex stalk, while \(\mW_2^{(\ell)}\in\R^{\ch\times\ch}\) mixes feature channels.

\subsection{Hierarchical Graph Pooling}
\label{subsec:hi_graph_pooling}

Let \(\mA\in\R_{+}^{\cN\times\cN}\) be the adjacency matrix of \(G\), and let \(\mH\in\R^{\cN\times\ch}\) be its node-feature matrix, where \(\vh_v\in\R^{\ch}\) denotes the representation of vertex \(v\). A pooling layer is a map \(\texttt{POOL}:(\mA,\mH)\longmapsto(\mA',\mH')\) that sends this graph to a coarser representation with \(\cK\leq\cN\) nodes, where \(\mA'\in\R^{\cK\times\cK}\) and \(\mH'\in\R^{\cK\times\ch}\). We describe graph pooling through the \emph{\gls{srcl}} framework \citep{grattarola2022understanding}, which decomposes the pooling operation into select, reduce, and connect, together with an optional lift. 
First, the selection function \texttt{SEL} assigns the \(\cN\) input nodes to \(\cK\) \emph{supernodes} and returns an assignment matrix according to \(\texttt{SEL}:(\mA,\mH)\mapsto \mS\in\R^{\cN\times\cK}\), where \(\ermS_{ia}\) denotes the membership of node \(i\) in supernode \(a\).
Membership can either be ``hard'' (each node is assigned to exactly one supernode) or ``soft'' (each node can belong to multiple supernodes with different membership scores, typically normalized to sum to one). 
Then, the reduction function \texttt{RED} aggregates the features of the nodes assigned to each supernode:
\(\texttt{RED}:(\mH,\mS)\mapsto\mH'\).
A standard sum-based reduction is \(\mH'=\mS^{\top}\mH\). 
Finally, the connection function \texttt{CON} constructs the connectivity of the pooled graph:
\(\texttt{CON}:(\mA,\mS)\mapsto\mA'\).
A common choice is \(\mA'=\mS^{\top}\mA\mS\). 
\rebuttal{For node-level tasks, pooling is paired with a lift \(\texttt{LIFT}:(\mH',\mS)\mapsto\widetilde{\mH}\) that maps coarse features back to the original \(\cN\) nodes, typically by broadcasting through \(\mS\) or its pseudoinverse so that \(\widetilde{\mH}\) lives in the same node space as \(\mH\).}
Torch Geometric Pool (\gls{tgp}) is a pooling library that implements this decomposition through a modular interface in which the \texttt{SEL}, \texttt{RED}, \texttt{CON}, and \texttt{LIFT} components can be defined and combined independently.

\section{Hierarchical Sheaf Pooling}
\label{sec:hierarchical_graph_pooling}

\subsection{Problem formulation and method overview}
\label{subsec:overview}

\begin{figure}[ht!]
    \centering
    \includegraphics[width=\columnwidth]{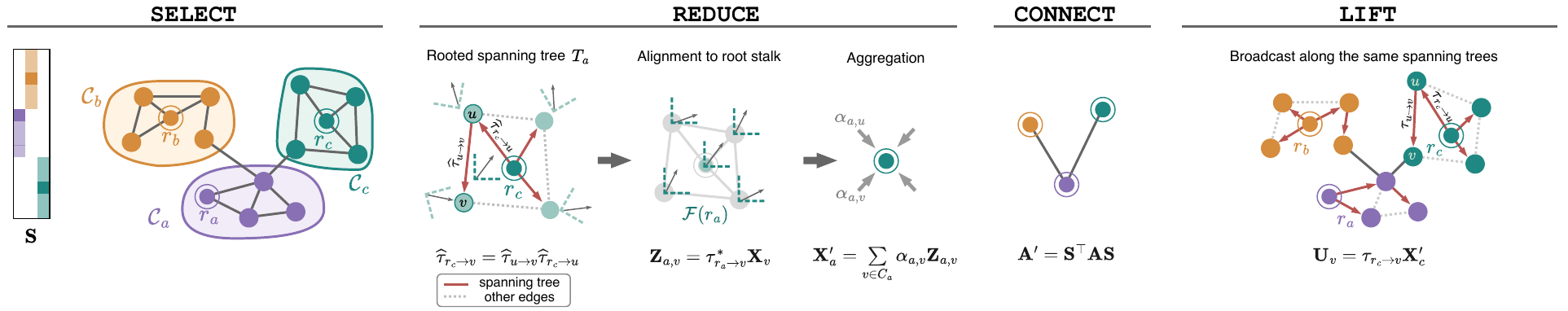}
    \caption{\gls{hisp} overview. \texttt{SEL} partitions the graph into connected clusters with their own root. \texttt{RED} aligns node features in the corresponding root stalk through tree-based transport and aggregates them into a coarse-node representation. \texttt{CON} constructs the coarse graph from the cluster assignments. \texttt{LIFT} restores node features to the fine graph.}
    \label{fig:sheafpool}
\end{figure}

Ordinary pooling operates on a shared feature matrix $\mH$. Here the input is a sheaf $\cF$ on $G=(\sV,\sE)$ together with stalk-valued features $\mX$\rebuttal{, each in its own stalk.}
\gls{hisp} extends the \gls{srcl} pipeline to this setting. At a high level it proceeds as follows.
\texttt{SEL} partitions the nodes into internally connected clusters and designates a root $r_a$ in each cluster $\sC_a$, which becomes the supernode of $\sC_a$ on the coarse graph.
Within each cluster it builds a rooted spanning tree; along the tree edges, the restriction maps of $\cF$ induce linear \emph{transport maps} that move features from one stalk into another.
\texttt{RED} uses these maps to pull every node feature into the common root stalk $\cF(r_a)$ and aggregates the aligned features into a single coarse representation.
\texttt{CON} then forms the coarse topology from inter-cluster connectivity.
The overall stage is summarized by
\[
\bigl(G,\cF,\mX\bigr)
\xrightarrow{\texttt{SEL}}
\bigl\{r_a\bigr\}_{a=1}^{\cK}
\xrightarrow{\texttt{RED}}
\mX'
\xrightarrow{\texttt{CON}}
G'.
\]
For node-level tasks, a dual \texttt{LIFT} broadcasts coarse features back to the fine graph along the same trees.
The subsequent sheaf layer learns new restriction maps on $G'$.
Figure~\ref{fig:sheafpool} illustrates the pipeline; the next subsections detail each stage.

\subsection{Select}
\label{subsec:select}

The \texttt{SEL} stage partitions the vertex set into \(\cK\) connected clusters, \(\sV=\sC_1\cup\cdots\cup \sC_{\cK}\). For each cluster \(\sC_a\), it chooses a root \(r_a\in \sC_a\), which becomes the \emph{supernode} representing $\sC_a$ in the coarse graph. Its stalk \(\cF(r_a)\) serves as the reference stalk for aggregating the features in \(\sC_a\). The coarse vertex set is therefore \(\sV'=\{r_a\}_{a=1}^{\cK}\). The partition is encoded by the hard assignment matrix \(\mS\in\{0,1\}^{\cN\times\cK}\), with \(\ermS_{va}=\mathbf{1}[v\in \sC_a]\).
\gls{hisp} can use any strategy to generate the assignments; it only requires that a rooted spanning tree can be constructed within each cluster.

We consider two topology-based strategies, chosen because their assignments depend only on the fixed graph topology \(\mA\) and can therefore be precomputed before training.
\gls{kmis} constructs a maximal set of well-separated representatives \citep{bacciu2023generalizing} and grows a connected cluster around each of them by multi-source \gls{bfs}.
Recursive spectral selection repeatedly bisects connected subgraphs at the median of their Fiedler vector \citep{pothen1990partitioning} and, in each final cluster, takes a maximum-degree node as the root.
Both methods produce connected clusters; their constructions are detailed in Appendix~\ref{app:select}.\rebuttal{We refer to the two selection variants using \gls{kmis} and recursive spectral selection as \gls{kmis}~\gls{hisp} and Fiedler~\gls{hisp}, respectively.}
\rebuttal{The root is that maximum-degree node, so transport is oriented toward a hub.}


The \texttt{SEL} stage returns the assignment matrix together with the roots:
\begin{equation}
    \label{eq:select}
    \texttt{SEL}(G) = \left( \mS, \{r_a\}_{a=1}^{\cK}\right).
\end{equation}

\subsection{Reduce}
\label{subsec:reduce}

\subsubsection{Transport}
\label{subsubsec:transport}

Let \(G[\sC_a]\) denote the subgraph of \(G\) induced by \(\sC_a\). Because \(\sC_a\) is connected, \(G[\sC_a]\) admits a spanning tree \(T_a\), which we root at \(r_a\).
In \(T_a\) each node \(v\in \sC_a\setminus\{r_a\}\) has a unique parent \(p_a(v)\) and a unique path from \(r_a\) to \(v\).
Consider a tree edge \(e=\{u,v\}\), oriented from the parent \(u=p_a(v)\) to the child \(v\).
The restriction maps \(\rho_{e,u}\) and \(\rho_{e,v}\) send both endpoint stalks to the common edge stalk \(\cF(e)\).
\rebuttal{A linear \emph{transport map} along this edge is a map
\begin{equation}
\label{eq:transport_map_def}
\tau_{u\to v}:\cF(u)\longrightarrow\cF(v)
\end{equation}
that moves features from the parent stalk to the child stalk through \(\cF(e)\): ideally,
\(\rho_{e,v}\tau_{u\to v}=\rho_{e,u}\), so transporting and then restricting recovers the parent's edge representation.}
Composing these maps along the unique tree path yields the root-to-node transport \(\tau_{r_a\to v}:\cF(r_a)\to\cF(v)\),
\begin{equation}
\label{eq:root_to_node_transport}
\tau_{r_a\to r_a}=\mI_{\cd},
\qquad
\tau_{r_a\to v} = \tau_{p_a(v)\to v} \tau_{r_a\to p_a(v)}.
\end{equation}
At the root the feature is unchanged; to reach any other node \(v\), we first transport it to its parent \(p_a(v)\), then apply the map from the parent to \(v\).

\rebuttal{When \(\rho_{e,v}\) is orthogonal, \(\rho_{e,v}^{-1}=\rho_{e,v}^{*}\) is exact; for general maps \(\rho_{e,v}\) may be singular, and for diagonal or conformal maps~\citep{ribeiro2025cooperative} entries near zero make exact inversion ill-conditioned.}
We therefore instantiate \(\tau_{u\to v}\) by Tikhonov-regularized least squares \citep{hansen1998rank},
\begin{equation}
\label{eq:transport_p_v}
\tau_{u\to v}
=
\left( \rho_{e,v}^{*}\rho_{e,v} +\varepsilon\mI \right)^{-1} \rho_{e,v}^{*}\rho_{e,u},
\qquad \varepsilon>0,
\end{equation}
which minimizes disagreement between the parent and child representations in \(\cF(e)\) while keeping the map well-defined when \(\rho_{e,v}\) is singular or ill-conditioned.
\rebuttal{Even with regularization, \(\|\tau_{u\to v}\|_2\) may exceed one for non-orthogonal maps.}
To keep feature norms from growing along tree paths, we replace each map by the normalized transport
$\widehat{\tau}_{u\to v}
=
\tau_{u\to v}
/
\max\{1,\|\tau_{u\to v}\|_2\}$,
so that $\|\widehat{\tau}_{u\to v}\|_2\leq 1$, and use \(\widehat{\tau}\) in Equation~\ref{eq:root_to_node_transport}.
Feature norms cannot grow along a path, though they may shrink with repeated application.
\rebuttal{Further details, specializations to orthogonal and conformal sheaves, and a stability ablation are provided in Appendix~\ref{app:non_expansive_transport}.}

\subsubsection{Aggregation}
\label{subsubsec:aggregation}

Aggregation requires every node feature to be expressed in the common root stalk \(\cF(r_a)\).
We keep the root-to-node transports of Section~\ref{subsubsec:transport} as the geometric frame of the cluster---the same maps later used by \texttt{LIFT}---and pull features back with their adjoints; alternative alignments are discussed in Appendix~\ref{app:non_expansive_transport}.
The adjoint \(\tau_{r_a\to v}^{*}:\cF(v)\to\cF(r_a)\) is defined by
\begin{equation}
\label{eq:transport_adjoint}
\langle \tau_{r_a\to v}\,\vx,\,\vx_v\rangle_v
=
\langle \vx,\, \tau_{r_a\to v}^{*}\,\vx_v\rangle_{r_a}
\qquad
\forall\,\vx\in\cF(r_a),\;
\vx_v\in\cF(v),
\end{equation}
so the aligned contribution of node \(v\) is
\[
\mZ_{a,v} = \tau_{r_a\to v}^{*}\mX_v \in
\cF(r_a)^\ch
\simeq
\R^{\cd\times\ch}.
\]
All features \(\{\mZ_{a,v}:v\in \sC_a\}\) now live in the same stalk and can be aggregated.
We compute a scalar score \( s_{a,v} = s_\theta\!\left(\operatorname{vec}(\mZ_{a,v})\right)\) with a learnable scoring function \(s_\theta\), and normalize within the cluster by
$\alpha_{a,v}=\exp(s_{a,v})/\sum_{u\in \sC_a}\exp(s_{a,u})$, so that $\sum_{v\in \sC_a}\alpha_{a,v}=1$.
The coarse feature at the supernode \(r_a\) is the resulting convex combination
\[
\mX'_{a}
=
\sum_{v\in \sC_a}
\alpha_{a,v}\mZ_{a,v}
=
\sum_{v\in \sC_a}
\alpha_{a,v}\tau_{r_a\to v}^{*}\mX_v
\in\cF(r_a)^\ch.
\]
Collecting the representations of all clusters gives \(\mX' = \{\mX'_{a}\}_{a=1}^{\cK}\), and the \texttt{RED} output is
\begin{equation}
    \label{eq:reduce}
    \texttt{RED}_{\sT} \left( \cF,\mX, \{\sC_a,r_a\}_{a=1}^{\cK} \right) = \mX'(\sT), 
\end{equation}
where \(\sT=\{T_a\}_{a=1}^{\cK}\) is the set of rooted spanning trees used for transport.

\subsubsection{Path dependence}
\label{subsubsec:path_dependence}

\rebuttal{Within a cluster, two different paths \(\pi_1\) and \(\pi_2\) stemming from distinct spanning trees, may connect the root \(r_a\) to the same node \(v\).}
Each path applies a different sequence of edge transport maps.
Since the restriction maps are learned, there is no guarantee that these sequences are equivalent, i.e., \( \tau_{r_a\to v}^{(\pi_1)} \neq \tau_{r_a\to v}^{(\pi_2)} \).
Consequently, the aligned feature \(\tau_{r_a\to v}^{(\pi)*}\mX_v\), and hence the resulting coarse-node representation, may depend on the path selected within the cluster.
To reduce the dependency to an arbitrary selection of the spanning tree, we resample \gls{bfs} spanning trees during training, exposing the model to multiple transport paths rather than a single fixed realization. 
Exposing the model to a distribution of valid trees mitigates the effect of the path dependence.
A detailed analysis and the corresponding stochastic objective are provided in Appendix~\ref{app:non_expansive_transport}.

\subsection{Connect}
\label{subsec:connect}
The \texttt{CON} stage defines the edges of the coarse graph $G'$ using the assignment matrix \(\mS\) from Equation~\ref{eq:select}:
\begin{equation}
    \label{eq:connect}
    \texttt{CON}(\mA,\mS) = \mA',\qquad
    \mA' = \mS^\top\mA\mS.
\end{equation}

For two distinct supernodes, the edge weight is the sum of the original edge weights between their clusters.
The restriction maps on \(G'\) are learned by the subsequent sheaf layer.

\begin{rebuttalblock}
\subsection{Lift}
\label{subsec:lift}

Node-level tasks require restoring coarse features to the fine graph.
\gls{hisp} does so by broadcasting each supernode representation back along the same rooted transport tree used by \texttt{RED}.
Given the coarse germ \(\mX'_{a}\in\cF(r_a)^\ch\) of cluster \(\sC_a\), the lifted feature at every node \(v\in\sC_a\) is
\begin{equation}
\label{eq:hisp_lift}
\mU_v
=
\tau_{r_a\to v}\mX'_{a}
\in\cF(v)^\ch.
\end{equation}
Thus \texttt{LIFT} applies the root-to-node maps \(\tau_{r_a\to v}\), while \texttt{RED} applies their adjoints \(\tau_{r_a\to v}^{*}\): the two operators are dual with respect to the stalk inner product and are taken with respect to the same tree collection \(\sT\).
Collecting over all nodes,
\begin{equation}
\label{eq:lift}
\texttt{LIFT}_{\sT}\!\left(\mX',\{\sC_a,r_a\}_{a=1}^{\cK}\right)
=
\mU(\sT).
\end{equation}\end{rebuttalblock}

\subsection{\texorpdfstring{\gls{hisp}}{HiSP} readout}
\label{subsec:hierarchical_sheaf_neural_networks}
For graph-level tasks, alternating the sheaf-diffusion layers of
Equation~\ref{eq:neural_sheaf_diffusion} with \gls{hisp} layers produces a sequence of progressively coarser graphs \( G_0\longrightarrow G_1\longrightarrow\cdots\longrightarrow G_{\cL}\).
The final graph \(G_{\cL}\) may contain multiple supernodes, either because coarsening stops before reaching a single node or because the input graph is disconnected. 
Their features belong to distinct root stalks and cannot be averaged directly. 
Therefore, we extract descriptors from each supernode, map them to a common embedding space, and aggregate them
there.
Specifically, for each final supernode \(r_a\in \sV_{\cL}\), representing a cluster \(\sC_a\subseteq \sV_{\cL-1}\), we compute three summaries. 
The first is a Gram descriptor \(\vphi_a^{(\mathrm{Gram})}=\mX_{r_a}^{(\cL)\top}\mX_{r_a}^{(\cL)}\), which records inner products between feature channels. 
The second is a stalk-coordinate mean \(\vphi_a^{(\mathrm{mean})}=\cd^{-1}\vone_{\cd}^\top\mX_{r_a}^{(\cL)}\), which averages each channel over the \(\cd\) stalk coordinates.
Finally, the cluster-statistics descriptor collects the pooled magnitude, the sheaf Dirichlet energy inside the cluster, and the magnitude before pooling:
\[
\vphi_a^{(\mathrm{stat})} = \left(
\left\|\mX_{r_a}^{(\cL)}\right\|_{\mathrm F},
\;
\operatorname{tr}\!\left(
\bigl(\widetilde{\mX}_{\sC_a}^{(\cL-1)}\bigr)^\top
\mL_a^{(\cL-1)}
\widetilde{\mX}_{\sC_a}^{(\cL-1)}
\right),
\;
\left\|\widetilde{\mX}_{\sC_a}^{(\cL-1)}\right\|_{\mathrm F} \right).
\]
Here, \(\widetilde{\mX}_{\sC_a}^{(\cL-1)}\) stacks the cluster's node features before the final pooling step, and \(\mL_a^{(\cL-1)}\) is the sheaf Laplacian restricted to edges within \(\sC_a\).
For each descriptor type \(b\in\sB=\{\mathrm{Gram},\mathrm{mean},\mathrm{stat}\}\), a learnable MLP \(\psi_b\), shared across all coarse nodes, maps the descriptor to a branch embedding; the Gram matrix is first flattened into a vector.
These embeddings are averaged over the final coarse nodes and concatenated before passing them to the MLP that computes the prediction $\widehat{\vy}_G$:

\begin{equation}
\label{eq:hisp_readout}
\mathbf r_G^{(b)}
=
\frac{1}{|\sV_{\cL}|}
\sum_{a\in \sV_{\cL}}
\psi_b\!\left(\vphi_a^{(b)}\right),
\quad
\mathbf g_G
=
\mathop{\mathbin{\Vert}}_{b\in\sB}
\mathbf r_G^{(b)},
\quad
\widehat{\vy}_G
=
\operatorname{MLP}_{\mathrm{out}}(\mathbf g_G).
\end{equation}

These components assemble into \glspl{hsnn} evaluated below.

\section{Experiments}
\label{sec:experiments}

A \gls{hsnn} alternates sheaf convolution with \gls{hisp}: each stage learns restriction maps at the current resolution, then pools into a coarser sheaf. Unless otherwise specified, we use Fiedler~\gls{hisp}. Graph-level models end with a \gls{hisp} readout; node-level models follow a pool--unpool strategy, lifting coarse features back to the fine graph along the same transport trees to form a sheaf U-Net.
Appendix~\ref{app:hierarchical_sheaf_architectures} details both constructions.

We evaluate whether this architecture can:
(i) learn independent sheaves at different resolutions;
(ii) transmit sheaf signals over long graph distances;
(iii) selectively isolate information relevant for the task;
(iv) improve long-range capabilities of a sheaf;
(v) leverage the hierarchical design to handle complex interactions among communities; and
(vi) support sheaf U-Nets and graph classification architectures.
Unless otherwise stated, results are reported as mean~$\pm$~standard deviation over five random seeds.

\paragraph{Baselines.}
\rebuttal{We compare against \gls{nsd}~\citep{bodnar2022neural}, \gls{gcn}, and \gls{gin}. For long-range node regression, we also evaluate BuNN~\citep{bamberger2025bundle} and CSNN~\citep{ribeiro2025cooperative}. We consider six pooling baselines implemented in \gls{tgp}, namely EigenPool~\citep{ma2019graph}, \gls{kmis}, TopKPool, SAGPool, MinCutPool, and DiffPool. For graph-level tasks, we compare the \gls{hisp} readout with global mean and SheafPool~\citep{braithwaite2024heterogeneous}. Unless otherwise specified, tables report the best-performing configuration in each GNN/pooler/readout family. Appendix~\ref{app:exp_details} provides the full results for the evaluated configurations.}

\subsection{Independent sheaves across resolutions}
\label{subsec:exp_independent_sheaves}

A \gls{hsnn} should learn a fine sheaf that is locally consistent inside communities and a coarse sheaf that captures a different kind of relationship after pooling. The point of these tasks is not that message passing is weak in general, but \rebuttal{that the graph label depends on coarse inter-community relationships that differ from the fine intra-community relationships on the input edges}: without an explicit hierarchy, a single fine sheaf cannot represent both levels. We probe this with the following graph-classification task.

\textbf{\gls{hierbi}.}
\rebuttal{Each graph has four rooms ($0,1$ on the left and $2,3$ on the right), with every left--right pair joined by a three-node corridor ($76$ nodes in total).} At the fine resolution, every room is a complete bipartite graph \(K_{8,8}\) whose two partitions encode the same room-level content with opposite feature signs, so averaging within a room cancels the signal. \rebuttal{Pooling each room into one supernode yields the coarse graph \(K_{2,2}\); node indices are randomly permuted so \texttt{SEL} sees only the unlabeled adjacency matrix.} Every room is assigned a hidden sign \(s_c\in\{+,-\}\), which is not provided as an explicit feature. Writing the four rooms in a fixed order, the graph receives label \(\mathrm A\) or \(\mathrm B\) according to the sign pattern:
\[
\text{A:}\quad (+,+,-,-),
\qquad
\text{B:}\quad (+,-,+,-).
\]
The construction is applied independently to two feature channels, whose combination produces one of the four following patterns:
\(\mathrm{AA}\), \(\mathrm{AB}\), \(\mathrm{BA}\), or \(\mathrm{BB}\).
The patterns define 4 classes, which are balanced and give chance accuracy of \(25\%\). 


Solving \gls{hierbi} requires a specific operation at each of the two resolutions. Within each room, the oppositely signed features must first be aligned before aggregation to avoid cancellation. 
The resulting four room representations must then be compared on the coarse graph to identify the right pattern. 
Dataset construction and training are detailed in Appendix~\ref{app:exp_hierbi}.
\rebuttal{We compared unpooled \gls{nsd} with a global mean, SheafPool, or the \gls{hisp} readout; and \gls{nsd} with each of the six poolers under either readout. \gls{nsd}+\gls{hisp} uses connected pooling at \(\cK{=}4\) and the \gls{hisp} readout. 
\gls{nsd}+\gls{hisp} reaches \(93.5\%\) while unpooled \gls{nsd} and every \gls{gcn}, \gls{gin}, and \gls{nsd}  pooler stay below \(50\%\).}
Section~\ref{sec:exp_hisp_ablation} and Table~\ref{tab:hierbi_ablation} separate the two \gls{hisp} pieces on this matched partition: sheaf transport is what preserves the room signal, and the \gls{hisp} readout is what reads the coarse stalks stably.

\begin{table}[h]
\centering
\caption{\rebuttal{HierBI test accuracy (\%, mean~$\pm$~std; $\uparrow$). Best of each family (Appendix~\ref{app:exp_hierbi}).}}
\label{tab:independent_sheaves}
\footnotesize
\setlength{\tabcolsep}{3pt}
\begin{tabular}{@{}ccccc@{}}
\toprule
\rebuttal{GNN+pool} & \rebuttal{\gls{nsd}} & \rebuttal{\gls{nsd}+pool} & \rebuttal{\gls{nsd} + \gls{hisp} readout} & \rebuttal{\gls{nsd}+\gls{hisp}} \\
\midrule
\rebuttal{$48.4{\pm}4.3$} & \rebuttal{$47.8{\pm}3.0$} & \rebuttal{$47.4{\pm}3.8$} & \rebuttal{$47.4{\pm}5.1$} & \rebuttal{$\mathbf{93.5{\pm}1.7}$} \\
\bottomrule
\end{tabular}
\end{table}

\subsection{Long-range signal transfer}
\label{subsec:exp_ringtransfer}

RingTransfer \citep{bodnar2021cwn} evaluates long-range signal transfer on a cycle with \(24\) vertices: a query node must recover the class label stored at the antipodal source node, \(12\) edges away.
The task has five classes, so chance accuracy is \(20\%\).
As shown in Table~\ref{tab:ringtransfer}, \gls{nsd} remains at chance for \(\cL\leq 6\). At \(\cL{=}12\), and in the best pooled cell, the target node is within reach, but neither result is reliable across seeds. The pooled cells are different combinations; the appendix grid names the winner.
By cutting the source--query distance to \(2\), \gls{nsd}+\gls{hisp} reaches \(100\%\) at \(\cL{=}3\) on every seed.
\rebuttal{Pooling operators able to generate arbitrary cluster assignments, such as MinCutPool and DiffPool, are not used as baselines because they can place the source and the query in the same cluster, bypassing the long-range problem instead of testing it.} 
Dataset construction, training, and the full grid are in Appendix~\ref{app:exp_ringtransfer}.

\begin{table}[h]
\centering
\caption{RingTransfer test accuracy (\%, mean~$\pm$~std; $\uparrow$).
\rebuttal{Best of each family (Appendix~\ref{app:exp_ringtransfer}).}}
\label{tab:ringtransfer}
\footnotesize
\setlength{\tabcolsep}{2pt}
\begin{tabular}{@{}l|cccccc@{}}
\toprule
Model & \rebuttal{GNN+pool} & \gls{nsd} ($\cL{=}3$) & \gls{nsd} ($\cL{=}6$) & \gls{nsd} ($\cL{=}12$) & \rebuttal{\gls{nsd}+pool} & \gls{nsd}+\gls{hisp} \\
\midrule
Accuracy & \rebuttal{$20.0{\pm}0.0$} & $20.0{\pm}0.0$ & $20.0{\pm}0.0$ & $72.0{\pm}34.9$ & \rebuttal{$84.0{\pm}32.0$} & $\mathbf{100.0{\pm}0.0}$ \\
\bottomrule
\end{tabular}
\end{table}


\subsection{Filtering capacity inside clusters}
\label{subsec:exp_filtering}

\textbf{Marked PathCopy.}
In this task, each sample is a path graph containing a marked source carrying one of five class labels, a marked query at a variable distance, and three marked distractors carrying unrelated labels~\citep{bianchi2026reachability}. 
We test whether \gls{hisp} can reproduce the source label at the query (i.e. learn to propagate the source signal while suppressing distractors). 
Dataset construction and training details are provided in Appendix~\ref{app:exp_marked_path}.
\rebuttal{MinCutPool and DiffPool are not used as baselines for the same motivation as in Section \ref{subsec:exp_ringtransfer}.}
As shown in Table~\ref{tab:filtering}, \gls{nsd}+\gls{hisp} reaches \(87.5\%\), \rebuttal{above unpooled \gls{nsd} (\(55.5\%\)) and the best \gls{gnn} and \gls{nsd} configurations with baseline poolers.}
On the connected partition, \gls{hisp} transport can down-weight on-path distractors, and coarsening shortens the source--query geodesic.

\begin{table}[h]
\centering
\caption{Marked PathCopy test accuracy (\%, mean~$\pm$~std.; $\uparrow$).
\rebuttal{Best of each family (Appendix~\ref{app:exp_marked_path}).}}
\label{tab:filtering}
\footnotesize
\setlength{\tabcolsep}{2pt}
\begin{tabular}{@{}l|cccc@{}}
\toprule
Model & \rebuttal{GNN+pool} & \gls{nsd} & \rebuttal{\gls{nsd}+pool} & \gls{nsd}+\gls{hisp} \\
\midrule
Accuracy & \rebuttal{$43.8{\pm}4.7$} & $55.5{\pm}4.1$ & \rebuttal{$45.4{\pm}4.0$} & $\mathbf{87.5{\pm}4.6}$ \\
\bottomrule
\end{tabular}
\end{table}


\subsection{Long-range node regression}
\label{subsec:exp_long_range}

\paragraph{\gls{lrim}.}

\begin{rebuttalblock}We evaluate inductive node regression on the \gls{lrim} benchmark \citep{mathys2026lrim} (\texttt{lrim\_16\_0.6\_10k}), whose targets provably require long-range interactions across the grid~\citep{tonshoff2023did,mathys2026lrim}.
We compare diagonal \gls{nsd}~\citep{bodnar2022neural}, BuNN~\citep{bamberger2025bundle}, and CSNN~\citep{ribeiro2025cooperative}, all with stalk dimension (d=2), with and without Fiedler~\gls{hisp} ((r=0.5)) across depths \(\cL\in\{4,8,12,16\}\) (Table~\ref{tab:lrim}; full grids in Appendix~\ref{app:exp_lrim}). The results shows that \gls{hisp} reduces the mean test error of diagonal \gls{nsd} at every tested depth, with larger gains in deeper models. For BuNN and CSNN, which already support longer-range propagation, the gains are smaller and vary across depths. \end{rebuttalblock}

\begin{table}[h]
\centering
\caption{LRIM-16-hard~\citep{mathys2026lrim} (test $\log_{10}$MSE $\downarrow$).
Each unpooled row is paired with Fiedler \gls{hisp} at $r{=}0.5$. Bold marks the best model.}
\label{tab:lrim}
\small
\begin{rebuttalblock}\begin{tabular}{@{}lcccc@{}}
\toprule
Model & $\cL{=}4$ & $\cL{=}8$ & $\cL{=}12$ & $\cL{=}16$ \\
\midrule
NSD             & $-0.68{\pm}.03$ & $-0.79{\pm}.28$ & $-0.61{\pm}.14$ & $-0.80{\pm}.19$ \\
\gls{nsd}+\gls{hisp}  & $\mathbf{-0.79{\pm}.04}$ & $\mathbf{-1.30{\pm}.09}$ & $\mathbf{-1.43{\pm}.06}$ & $\mathbf{-1.50{\pm}.06}$ \\
\midrule
BuNN            & $-1.18{\pm}.05$ & $-1.61{\pm}.07$ & $\mathbf{-1.75{\pm}.06}$ & $-1.81{\pm}.05$ \\
BuNN+\gls{hisp} & $\mathbf{-1.32{\pm}.06}$ & $\mathbf{-1.66{\pm}.03}$ & $-1.73{\pm}.05$ & $\mathbf{-1.87{\pm}.04}$ \\
\midrule
CSNN            & $-1.56{\pm}.02$ & $\mathbf{-1.79{\pm}.02}$ & $\mathbf{-1.87{\pm}.02}$ & $-1.85{\pm}.04$ \\
CSNN+\gls{hisp} & $\mathbf{-1.62{\pm}.02}$ & $-1.78{\pm}.03$ & $\mathbf{-1.87{\pm}.03}$ & $\mathbf{-1.91{\pm}.02}$ \\
\bottomrule
\end{tabular}\end{rebuttalblock}
\end{table}

\subsection{Real-world coalition reading: VoteView}
\label{subsec:exp_voteview}

We evaluate sheaf-aware hierarchy on coalition classification using U.S.\ Senate roll calls from Congresses~90--100, obtained from VoteView~\citep{lewis2026voteview}.

Each Congress defines a fixed graph whose nodes represent senators. 
We divide each party into two ideological wings using NOMINATE~\citep{poole1985spatial} positions, yielding four communities.
Within each party, an ideological \(k\)-nearest-neighbor graph is augmented with a minimum spanning tree; dummy corridor nodes connect the parties across the aisle (Figure~\ref{fig:voteview_hemicycle}). 

Each roll call defines an instance on this graph. 
To assign its label, we compute the majority ballot sign in each wing and normalize the four signs by a global flip when necessary so that the first wing is positive.
We retain the party-line \((+,+,-,-)\) and crossed \((+,-,+,-)\) patterns, yielding \(1012\) instances balanced between the two classes, with chance accuracy \(50\%\). 
Ballots are used only to determine labels. 
\rebuttal{The input feature for senator \(i\) is \(\vx_i=\vd_i-\vm\), where \(\vd_i\in\R^2\) is the senator's static NOMINATE coordinate and \(\vm\in\R^2\) is the ideological midpoint of the bill under debate; corridor nodes use the same construction with their layout coordinates.
Models must thus infer the coalition from ideological positions relative to the bill topic without access to individual votes.} 

\begin{wrapfigure}{r}{0.50\columnwidth}
    \centering
    \includegraphics[width=\linewidth]{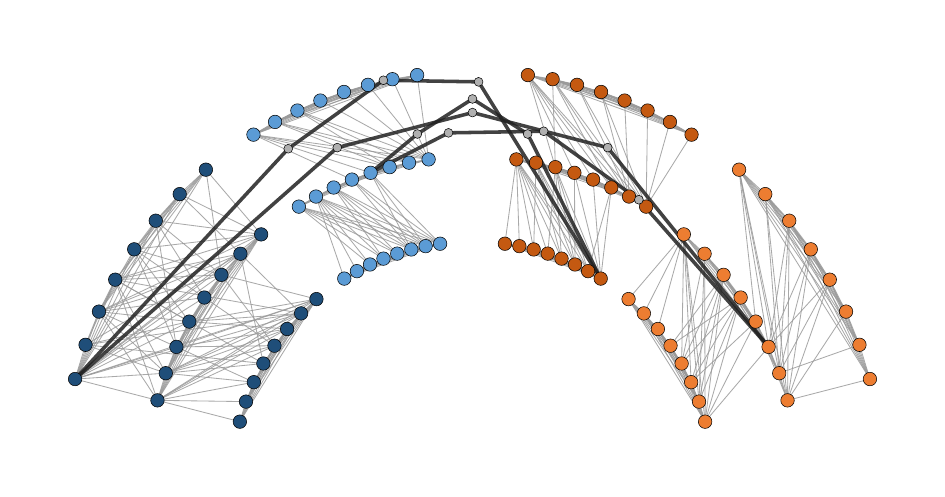}
    \caption{VoteView graph in a hemicycle layout. Blue and orange blocks
    denote the Democratic and Republican ideological wings, respectively;
    grey nodes form corridors across the aisle.}
    \label{fig:voteview_hemicycle}
\end{wrapfigure}

To test sensitivity to within-wing feature coherence, we independently retain each senator's feature vector \(c\) and negate it with probability \(1-c\). Inconsistency is greatest near \(c=0.5\); as \(c\) approaches zero, the features instead approach a coherent global sign reversal. Accuracy therefore need not vary monotonically with \(c\). \rebuttal{Negating a senator's feature breaks agreement inside a wing and removes that hierarchy. Table \ref{tab:voteview_coherence} shows that \gls{nsd}+\gls{hisp} achieves the highest mean test accuracy at \(c=1.0\) and \(c=0.1\), when within-wing feature alignment is largely preserved. At \(c=0.7\) and \(c=0.4\), where original and negated features are more mixed, unpooled \gls{nsd} with the \gls{hisp} readout achieves the highest mean accuracy. This pattern suggests that hierarchical pooling is most effective when features are coherent within wings, while the \gls{hisp} readout can support classification directly from senator-level representations when that coherence weakens} Dataset construction, training, and every run are in Appendix~\ref{app:exp_voteview}.

\begin{table}[h]
\centering
\caption{VoteView test accuracy (\%, mean\(\pm\)std; $\uparrow$). \rebuttal{Best of each family (Appendix~\ref{app:exp_voteview}).}}
\label{tab:voteview_coherence}
\footnotesize
\setlength{\tabcolsep}{3pt}
\begin{tabular}{@{}lcccc@{}}
\toprule
Model & \(c{=}1.0\) & \(c{=}0.7\) & \(c{=}0.4\) & \(c{=}0.1\) \\
\midrule
\rebuttal{GNN+pool} & \rebuttal{$85.4{\pm}0.9$} & \rebuttal{$76.9{\pm}4.3$} & \rebuttal{$76.3{\pm}1.3$} & \rebuttal{$83.8{\pm}0.8$} \\
\rebuttal{\gls{nsd}} & \rebuttal{$83.0{\pm}2.7$} & \rebuttal{$80.7{\pm}1.5$} & \rebuttal{$83.5{\pm}3.5$} & \rebuttal{$80.8{\pm}3.4$} \\
\rebuttal{\gls{nsd}+pool} & \rebuttal{$83.8{\pm}2.0$} & \rebuttal{$83.3{\pm}1.3$} & \rebuttal{$84.7{\pm}1.3$} & \rebuttal{$83.8{\pm}2.4$} \\
\rebuttal{\gls{nsd} + \gls{hisp} readout} & \rebuttal{$85.0{\pm}1.7$} & \rebuttal{$\mathbf{84.3{\pm}1.6}$} & \rebuttal{$\mathbf{85.5{\pm}1.5}$} & \rebuttal{$82.5{\pm}1.1$} \\
\gls{nsd}+\gls{hisp} & \rebuttal{$\mathbf{88.5{\pm}1.9}$} & \rebuttal{$83.3{\pm}1.6$} & \rebuttal{$83.0{\pm}3.1$} & \rebuttal{$\mathbf{83.9{\pm}1.5}$} \\
\bottomrule
\end{tabular}
\end{table}

\subsection{Standard benchmarks}
\label{subsec:exp_standard}
For the following experiments, we include the \gls{hisp} selection strategy in the hyperparameter sweep.
\paragraph{Node classification.}
We combine one or two \gls{hisp} layers with \gls{nsd} backbones to build sheaf U-Nets for node-level tasks.
On each of the nine datasets, at least one of the one- and two-stage sheaf U-Nets achieves a mean test accuracy within approximately \(1.6\) percentage points of unpooled \gls{nsd}. The largest improvements are on Chameleon (\(+1.3\) points) and Squirrel (\(+1.0\) point), as shown in Table~\ref{tab:node_class_full}.
A sheaf U-Net can therefore insert \gls{hisp} and keep the accuracy of flat \gls{nsd}. The clearest gains are on the heterophilic graphs, Chameleon (\(+1.3\)) and Squirrel (\(+1.0\)).
Results and details of the evaluation protocol are in Appendix~\ref{app:node_class}.


\paragraph{Graph classification.}
On MUTAG, PROTEINS, ENZYMES, and NCI1, we compare unpooled \gls{nsd} and \gls{nsd}+\gls{hisp}, each paired with either SheafPool or the \gls{hisp} readout (Table~\ref{tab:graph_level}).
Pooling and readout must work together: \gls{hisp} represents each cluster in its root stalk, so different supernodes need not share a common coordinate system. 
A readout that assumes shared coordinates or relies on the original fine sheaf may therefore fail to use the coarse representations effectively.
The \gls{hisp} readout extracts descriptors from each node or supernode, maps them to a common embedding space, and aggregates the embeddings into a graph representation. 
It can thus be used both with unpooled \gls{nsd} and after \gls{hisp} pooling.

The benefits depend on the dataset. 
On PROTEINS, combining \gls{hisp} pooling and readout yields the highest mean accuracy; on NCI1, the readout alone achieves comparable performance. 
On ENZYMES, the highest mean accuracy is obtained without pooling, suggesting that coarsening may remove useful fine-scale information. 
Tuning and evaluation details are provided in Appendix~\ref{app:exp_graph}.

\begin{table}[h]
\centering
\caption{Graph classification (fold-mean test accuracy, \%;  $\uparrow$).}
\label{tab:graph_level}
\small
\begin{tabular}{lcccc}
\toprule
Model (Readout) & MUTAG & PROTEINS & ENZYMES & NCI1 \\
\midrule
\gls{nsd} (SheafPool)
  & $82.0{\pm}6.2$
  & $65.3{\pm}5.3$
  & $39.5{\pm}7.4$
  & $49.9{\pm}0.1$ \\
\gls{nsd} (\gls{hisp} readout)
  & $79.9 {\pm} 11.2$
  & $69.9 {\pm} 4.2$
  & $\mathbf{41.0 {\pm} 5.7}$
  & $\mathbf{74.6 {\pm} 4.1}$ \\
\gls{nsd}+\gls{hisp} (SheafPool)
  & $80.0{\pm}11.7$
  & $69.8{\pm}4.0$
  & $26.8{\pm}4.7$
  & $50.0{\pm}0.1$ \\
\gls{nsd}+\gls{hisp} (\gls{hisp} readout)
  & $\mathbf{82.5{\pm}5.1}$
  & $\mathbf{74.4{\pm}4.1}$
  & $37.8{\pm}6.4$
  & $74.2{\pm}3.0$ \\
\bottomrule
\end{tabular}
\end{table}

\subsection{Ablations}
\label{sec:exp_hisp_ablation}
\label{sec:exp_tree}

\rebuttal{We ablate each component of \gls{hisp} and report the full results in Appendix~\ref{app:ablations}. In particular, we conduct targeted experiments to isolate the effects of the sheaf-aware components. On \gls{hierbi}, with the partition and \gls{hisp} readout fixed, aligning node features before averaging raises accuracy from 46.2\% to 93.5\% (Table~\ref{tab:hierbi_ablation}), consistent with alignment preventing cancellation of oppositely signed features within each room. Once the room signal is preserved, a global-mean readout already reaches 89.0\%, while the \gls{hisp} readout improves accuracy and reduces variability across seeds. To further assess the role of sheaf-aware alignment in \texttt{RED}, we compare adjoint and damped pseudoinverse transport with a shared MLP applied before pooling. Both transport variants outperform the shared MLP on \gls{hierbi} and Marked PathCopy (Table~\ref{tab:transport_pullback}). On \gls{hierbi}, the shared MLP performs similarly to direct mean pooling, which is equivalent to using identity transport. The strong performance of both sheaf-based variants points to the value of using restriction maps for alignment. The choice of root also matters for alignment. On \gls{hierbi}, we find that accuracy is more sensitive to root selection than to tree sampling. With the maximum-degree root fixed, accuracy is similar across the tested tree-sampling strategies. Replacing it with a minimum-degree or random root substantially lowers accuracy, supporting the use of a room hub as the reference for alignment. Further experiments on the selection strategy, feature noise, readout descriptors, and transport operator-norm clipping are in Appendix~\ref{app:ablations}.}

\subsection{Implementation}
\label{subsec:exp_implementation}

Every experiment uses a \gls{pyg} message-passing port of \gls{nsd}; Appendix~\ref{app:sheaf_unet} shows that it matches the original implementation to numerical precision.
Training runs in \gls{tgp}, the PyTorch library for pooling in \glspl{gnn}.
For poolers whose assignments depend only on the fixed graph topology---including the \gls{kmis} and recursive spectral selectors used here---\texttt{PreCoarsening} builds and stores the hierarchy once, and \texttt{PoolDataLoader} collates those levels into mini-batches.
We extend this pipeline to cache the index structure of each sheaf Laplacian and the rooted cluster trees required by \gls{hisp}.
During batching, sheaf indices are shifted by the corresponding numbers of edges, nodes, or stalk coordinates, while tree parents are stored as relative node indices.
Structural fields are therefore computed once per graph and reused across epochs, \rebuttal{while each forward pass inverts only $|\sV|-\cK$ tree edges of small stalk dimension $\cd$, yielding a substantial speed-up and lower peak training memory on larger graphs (Appendices~\ref{app:select},~\ref{app:forward_timing}, and~\ref{app:sheaf_cache})}.

\section{Conclusions}
\label{sec:conclusions}

\rebuttal{\glspl{snn} use restriction maps to learn how neighboring node representations should be compared. \gls{hisp} extends this approach to hierarchical representation learning by aligning node features in a common cluster stalk before aggregation and learning new restriction maps on the coarse graph. Targeted experiments show that this hierarchy can improve the performance compared to flat \glspl{snn} in tasks that depend on group-level patterns or long-range interactions. By providing a way to pool sheaf representations, \gls{hisp} opens new directions for research on hierarchical \glspl{snn}. A natural next step is to use the learned sheaf to guide other stages of coarsening, from forming clusters to deciding how they connect on the coarse graph. Another is to characterize which fine-scale relations are preserved under coarsening and what the newly learned coarse sheaf reveals about interactions between groups.}

\bibliographystyle{unsrtnat}
\bibliography{sheafpool}

\appendix

\section*{Appendix}

\section{Selection strategies}
\label{app:select}

Both selection strategies depend only on the topology of the input graph.
Their cluster assignments and representatives can therefore be computed once
before training and reused at every forward pass.

\paragraph{\gls{kmis}.}
Given an integer \(k\), \gls{kmis} constructs a maximal \(k\)-independent set
\(R\subseteq \sV\). Its representatives are mutually separated,
\[
    \operatorname{dist}_G(r_i,r_j)>k
    \qquad
    \text{for all distinct } r_i,r_j\in R,
\]
while maximality ensures that every vertex lies within \(k\) hops of at least
one representative. Vertices are assigned to nearby representatives through
a multi-source graph traversal, with ties resolved deterministically. This
produces a hard partition
\(\sC=\{\sC_1,\ldots,\sC_{\cK}\}\), where each cluster \(\sC_a\) contains its
representative \(r_a\) and is connected.

\paragraph{Recursive spectral selection.}
Starting from each connected component of \(G\), recursive spectral selection
computes the Fiedler vector of the Laplacian of the current induced subgraph
and bisects its vertices at the median Fiedler coordinate. If either side is
disconnected, it is separated into its connected components before recursion
continues. The procedure recursively splits the largest remaining connected subsets until the target number of clusters is reached. For each final cluster \(\sC_a\),
we select
\[
    r_a
    \in
    \operatorname*{arg\,max}_{v\in \sC_a}
    \deg_{G[\sC_a]}(v)
\]
as its representative, breaking ties deterministically.

\paragraph{Output.}
Both procedures return a hard assignment matrix
\[
    \mS\in\{0,1\}^{|\sV|\times\cK},
    \qquad
    \sC_a=\{v\in \sV:\ermS_{va}=1\},
\]
whose columns induce \(\cK\) connected clusters, together with a root
\(r_a\in \sC_a\) for each cluster. The roots initialize the transport trees used
by \texttt{RED}.

\begin{rebuttalblock}\paragraph{Computational complexity and permutation invariance.}
Both selection procedures depend solely on the static topology of \(G\) and are executed once before training via \texttt{SheafNetPretransform} (Appendix~\ref{app:sheaf_cache}).
In preprocessing, \gls{kmis} with degree scoring and multi-source \gls{bfs} runs in \(O(\kappa(|\sV|+|\sE|))\) time, while recursive spectral bisection computes \(\lceil\log_2\cK\rceil\) sparse Fiedler splits in \(O((|\sV|+|\sE|)\log\cK)\) time; intra-cluster \gls{bfs} spanning-tree construction and maximum-degree root selection take \(O(|\sV|+|\sE|)\) time.
At training and inference time, cluster assignments and cached \gls{bfs} parent pointers are retrieved in \(O(1)\) time (or \(O(|\sV|+|\sE_{\mathrm{cluster}}|)\) when resampling random-neighbor \gls{bfs} trees), and tree transport with \(\cd\times\cd\) restriction maps requires \(O((|\sV|-\cK)\cd^3+|\sV|\cd^2\ch)\) operations since only the \(|\sV|-\cK\) tree edges are inverted.
Because both selectors and maximum-degree root selection are defined purely by graph topology and vertex degrees, \texttt{SEL} is permutation-equivariant (and graph-level predictions are permutation-invariant) up to deterministic tie-breaking of equal-degree automorphic vertices.\end{rebuttalblock}

\section{Regularized and Non-Expansive Transport Maps}
\label{app:non_expansive_transport}

This section provides further details on the transport maps introduced in Section~\ref{subsubsec:transport}. We derive the edgewise transport from a regularized least-squares problem, explain why regularization alone does not prevent pathwise amplification, and establish the stability of the normalized transport used by \gls{hisp}. As in the main text, we omit the layer index \(\ell\).

\paragraph{Edgewise transport.}
Consider a cluster \(\sC_a\) with root \(r_a\), and let \(T_a\) be a spanning tree of the induced subgraph \(G[\sC_a]\). Rooting \(T_a\) orients every tree edge from parent to child. For an edge \(e=\{u,v\}\), let \(u=p_a(v)\) be the parent of \(v\). The restriction maps
\[
\rho_{e,u}:\cF(u)\longrightarrow\cF(e),
\qquad
\rho_{e,v}:\cF(v)\longrightarrow\cF(e)
\]
map the endpoint stalks to the common edge stalk.

Given \(\vx_u\in\cF(u)\), its representation in \(\cF(e)\) is \(\rho_{e,u}\vx_u\). To obtain a corresponding
vector in \(\cF(v)\), we solve
\begin{equation}
\label{eq:app_transport_least_squares}
\vy^\star
=
\argmin_{\vy\in\cF(v)}
\left\|
\rho_{e,v}\vy-\rho_{e,u}\vx_u
\right\|_2^2
+
\varepsilon\|\vy\|_2^2,
\qquad
\varepsilon>0.
\end{equation}
The first term measures disagreement in the edge stalk, while the second penalizes solutions with large norm. The corresponding normal equations are
\[
\left(
\rho_{e,v}^{*}\rho_{e,v}
+\varepsilon\mI_{\cd}
\right)\vy^\star
=
\rho_{e,v}^{*}\rho_{e,u}\vx_u.
\]
Since \(\varepsilon>0\), the matrix on the left-hand side is positive definite and therefore invertible. Consequently,
\[
\vy^\star
=
\left(
\rho_{e,v}^{*}\rho_{e,v}
+\varepsilon\mI_{\cd}
\right)^{-1}
\rho_{e,v}^{*}\rho_{e,u}\vx_u.
\]
This defines the Tikhonov-regularized transport
\citep{hansen1998rank}
\begin{equation}
\label{eq:app_edge_transport}
\tau_{u\to v}
=
\left(
\rho_{e,v}^{*}\rho_{e,v}
+\varepsilon\mI_{\cd}
\right)^{-1}
\rho_{e,v}^{*}\rho_{e,u}
:
\cF(u)\longrightarrow\cF(v).
\end{equation}

\paragraph{Regularization does not imply non-expansiveness.}
The regularized inverse
\[
\rho_{e,v}^{+,\varepsilon}
=
\left(
\rho_{e,v}^{*}\rho_{e,v}
+\varepsilon\mI_{\cd}
\right)^{-1}
\rho_{e,v}^{*}
\]
is well defined even when \(\rho_{e,v}\) is singular. If \(\sigma\) is a singular value of \(\rho_{e,v}\), the corresponding singular
value of \(\rho_{e,v}^{+,\varepsilon}\) is
\[
g_\varepsilon(\sigma)
=
\frac{\sigma}{\sigma^2+\varepsilon}.
\]
Since
\[
\max_{\sigma\geq 0}g_\varepsilon(\sigma)
=
\frac{1}{2\sqrt{\varepsilon}},
\]
the regularization bounds the amplification due to the inverse, but this bound may still be greater than one. In particular,
\[
\|\tau_{u\to v}\|_2
\leq
\frac{\|\rho_{e,u}\|_2}{2\sqrt{\varepsilon}},
\]
which does not generally imply
\(\|\tau_{u\to v}\|_2\leq 1\).

This distinction becomes important when transports are composed. Let
\[
\pi_{r_a\to v}
=
(v_0=r_a,v_1,\ldots,v_m=v)
\]
be the unique path from \(r_a\) to \(v\) in \(T_a\). Without normalization, the corresponding root-to-node transport is
\[
\tau_{r_a\to v}
=
\tau_{v_{m-1}\to v_m}
\cdots
\tau_{v_0\to v_1}.
\]
By submultiplicativity of the operator norm,
\begin{equation}
\label{eq:app_unnormalized_path_bound}
\|\tau_{r_a\to v}\|_2
\leq
\prod_{j=1}^{m}
\|\tau_{v_{j-1}\to v_j}\|_2.
\end{equation}
Thus, if every edge transport has norm at most \(1+\delta\), the worst-case bound grows as \((1+\delta)^m\). Even moderate edgewise amplification may therefore accumulate along long paths, making the scale of the pooled features
depend strongly on the cluster diameter and potentially destabilizing activations and gradients.

\paragraph{Non-expansive normalization.}
To prevent pathwise amplification, we rescale every edge transport as
\begin{equation}
\label{eq:app_normalized_transport}
\widehat{\tau}_{u\to v}
=
\frac{\tau_{u\to v}}
{\max\left\{1,\|\tau_{u\to v}\|_2\right\}},
\qquad
\|\widehat{\tau}_{u\to v}\|_2\leq 1.
\end{equation}
This operation leaves \(\tau_{u\to v}\) unchanged when it is already
non-expansive and rescales it otherwise. Because the entire map is multiplied
by a scalar, the normalization preserves its rank, singular vectors, and
relative singular values. It therefore retains the anisotropic structure
encoded by the learned restriction maps.

This normalization only rescales the transport; it does not make it norm-preserving. Directions that were attenuated remain attenuated, and the kernel of the transport is unchanged. It therefore prevents amplification without recovering information already discarded by the transport.

\paragraph{Composition along the rooted tree.}
The normalized root-to-node transport
\[
\tau_{r_a\to v}:
\cF(r_a)\longrightarrow\cF(v)
\]
is defined recursively by
\begin{equation}
\label{eq:app_root_to_node_transport}
\tau_{r_a\to r_a}=\mI_{\cd},
\qquad
\tau_{r_a\to v}
=
\widehat{\tau}_{p_a(v)\to v}
\tau_{r_a\to p_a(v)}.
\end{equation}
Because \(T_a\) contains a unique path from \(r_a\) to every \(v\in \sC_a\), this recursion is well defined. Moreover,
\begin{equation}
\label{eq:app_non_expansive_path_transport}
\|\tau_{r_a\to v}\|_2
\leq
\prod_{(u\to w)\in\pi_{r_a\to v}}
\|\widehat{\tau}_{u\to w}\|_2
\leq 1.
\end{equation}
Hence, the complete root-to-node transport is non-expansive. Since the operator norm is invariant under adjunction,
\[
\|\tau_{r_a\to v}^{*}\|_2
=
\|\tau_{r_a\to v}\|_2
\leq 1.
\]
It follows that, for vectors
\(\vz\in\cF(r_a)\) and
\(\vx_v\in\cF(v)\),
\begin{equation}
\label{eq:app_transport_stability}
\|\tau_{r_a\to v}\vz\|_2
\leq
\|\vz\|_2,
\qquad
\|\tau_{r_a\to v}^{*}\vx_v\|_2
\leq
\|\vx_v\|_2.
\end{equation}
For matrix-valued node features \(\mX_v\in\R^{\cd\times\ch}\), the analogous Frobenius-norm bound is
\[
\|\tau_{r_a\to v}^{*}\mX_v\|_{\mathrm F}
\leq
\|\mX_v\|_{\mathrm F}.
\]
The same stability guarantee therefore applies both to the root-to-node transport used for lifting and to the aligned feature
\[
\mZ_{a,v}
=
\tau_{r_a\to v}^{*}\mX_v
\]
used for pooling.

\paragraph{Attenuation and path dependence.}
Non-expansiveness prevents exploding feature norms, but products of transports with norm strictly below one may attenuate features along long paths. This motivates selecting compact clusters, since shorter root-to-node paths reduce
both amplification before normalization and attenuation after normalization.

The construction also depends on the selected root and spanning tree. For a tree edge \(e=\{u,v\}\), the reverse regularized transport is
\[
\tau_{v\to u}
=
\left(
\rho_{e,u}^{*}\rho_{e,u}
+\varepsilon\mI_{\cd}
\right)^{-1}
\rho_{e,u}^{*}\rho_{e,v}.
\]
In general,
\[
\tau_{v\to u}
\neq
\tau_{u\to v}^{-1}
\qquad\text{and}\qquad
\tau_{v\to u}
\neq
\tau_{u\to v}^{*}.
\]
Therefore, applying \(\tau_{r_a\to v}^{*}\) to align a node feature with the root is not equivalent to composing independently computed reverse least-squares transports. Changing the root may reverse some edge orientations,
while changing the spanning tree may select different paths, leading to a different pooled representation. Sampling roots and spanning trees can mitigate this dependence by exposing the model to multiple orientations and transport paths.

\paragraph{Stochastic tree resampling.}
To reduce sensitivity to a single path system, we may keep the clusters and
roots fixed while resampling the BFS spanning trees at each forward pass. Let
\(\mathcal P\) denote the resulting parameter-independent distribution over
tree collections \(\sT=\{T_a\}_{a=1}^{\cK}\). Training then minimizes the expected
loss
\[
J(\theta)
=
\mathbb E_{\sT\sim\mathcal P}
\left[
\mathcal L\bigl(\widehat{\vy}(\sT;\theta),\vy\bigr)
\right].
\]
Under the usual conditions for interchanging differentiation and expectation,
a tree collection sampled at each forward pass gives an unbiased stochastic gradient:
\[
\mathbb E_{\sT\sim\mathcal P}
\left[
\nabla_\theta
\mathcal L\bigl(\widehat{\vy}(\sT;\theta),\vy\bigr)
\right]
=
\nabla_\theta J(\theta).
\]

At evaluation, the pooled representation can be averaged over \(M\)
independent tree-collection samples:
\[
\overline{\mX}_c^{(M)}
=
\frac{1}{M}
\sum_{m=1}^{M}
\mX_c\!\left(\sT^{(m)}\right),
\qquad
\sT^{(m)}\overset{\mathrm{i.i.d.}}{\sim}\mathcal P.
\]
This estimator approximates
\(\mathbb E_{\sT}[\mX_c(\sT)]\) as \(M\) increases. Because the downstream
predictor is nonlinear, applying it to the averaged representation is not, in
general, equivalent to averaging the predictions obtained from individual
trees. Tree resampling therefore marginalizes over path choices rather than
making the transport path-independent.

\begin{rebuttalblock}\paragraph{Specialization to flat vector bundles (BuNN) and conformal sheaves (CSNN).}
When the sheaf is a flat vector bundle as in BuNN~\citep{bamberger2025bundle}, each vertex \(v\in\sV\) carries a single orthogonal map \(\bm{O}_v\in O(\cd)\) shared across its incident edges (\(\rho_{e,v}=\bm{O}_v\)).
The edge transport in Equation~\ref{eq:app_edge_transport} (at \(\varepsilon=0\)) simplifies to \(\tau_{u\to v}=\bm{O}_v^\top\bm{O}_u\in O(\cd)\) with \(\|\tau_{u\to v}\|_2=1\), and composing edge transports along any path \(\pi\) from \(r_a\) to \(v\) telescopes via \(\bm{O}_{v_k}\bm{O}_{v_k}^\top=\mI_{\cd}\) to
\[
\tau_{r_a\to v}=\bm{O}_v^\top\bm{O}_{r_a},
\qquad
\tau_{r_a\to v}^{*}=\bm{O}_{r_a}^\top\bm{O}_v=\tau_{r_a\to v}^{-1}.
\]
Thus, on flat vector bundles, \gls{hisp} transport is strictly path-independent (requiring no spanning tree or tree resampling), isometric (requiring no operator-norm clipping and incurring zero path attenuation), and computable in \(O(|\sV|\cd^2\ch)\) parallel operations.

For directed conformal sheaves in CSNN~\citep{ribeiro2025cooperative}, where each node has source and target maps \(\bm{S}_u=C_{\bm{S}_u}\bm{Q}_u\) and \(\bm{T}_v=C_{\bm{T}_v}\bm{R}_v\) with \(\bm{Q}_u,\bm{R}_v\in O(\cd)\) and \(C_{\bm{S}_u},C_{\bm{T}_v}\ge 0\), the regularized normalized edge transport reduces in closed form to
\[
\widehat{\tau}_{u\to v}
=
\min\!\left\{1,\,\frac{C_{\bm{T}_v}C_{\bm{S}_u}}{C_{\bm{T}_v}^2+\varepsilon}\right\}\bm{R}_v^\top\bm{Q}_u,
\]
which requires no matrix inversion and smoothly sends \(\widehat{\tau}_{u\to v}\to\vzero\) whenever \(C_{\bm{T}_v}\to 0\) (preserving CSNN's cooperative \texttt{LISTEN}/\texttt{ISOLATE} gating). For a fixed positive source scale the unregularized inverse \(\bm{T}_v^{-1}\bm{S}_u\propto C_{\bm{S}_u}/C_{\bm{T}_v}\) would diverge as \(C_{\bm{T}_v}\to 0\).

\paragraph{Operator-norm clipping ablation.}
Table~\ref{tab:clipping_ablation} (Appendix~\ref{app:ablations}) compares \gls{hisp} with and without the non-expansive normalization \(\widehat{\tau}_{u\to v}\) (Equation~\ref{eq:app_normalized_transport}) across diagonal, general, and orthogonal (\texttt{bundle}) restriction maps.
For orthogonal maps, \(\|\tau_{u\to v}\|_2=(1+\varepsilon)^{-1}<1\), so the clip divides by \(1\) and leaves the transport unchanged.
Marked PathCopy therefore matches exactly with and without clipping.
The HierBI bundle means differ by \(0.3\) points, which sits inside the seed standard deviations (\(2.5\) and \(0.9\)): that gap is run-to-run variation, not an effect of clipping.
For non-orthogonal restriction maps, where small singular values in \(\rho_{e,v}\) can produce edge transports with \(\|\tau_{u\to v}\|_2>1\) that compound along tree paths and across pooling stages, removing operator-norm clipping degrades accuracy by \(8.7\) percentage points on HierBI (\texttt{diag}) and by \(18.8\) percentage points on the two-stage Marked PathCopy U-Net (\texttt{general}).

\end{rebuttalblock}

\section{Hierarchical Sheaf Neural Network Architectures}
\label{app:hierarchical_sheaf_architectures}

We refer to an architecture that alternates sheaf convolution and \gls{hisp} layers as a \gls{hsnn}. At hierarchy level \(\ell\), a sheaf-convolution block with parameters \(\theta_\ell\) learns the restriction maps
\(\boldsymbol{\rho}^{(\ell)}\) and updates the node features to \(\widetilde{\mX}^{(\ell)}\). The subsequent \gls{hisp} layer uses these restriction maps to align and aggregate the updated features:
\begin{equation}
\label{eq:hierarchical_sheaf_stage}
\bigl(G_\ell,\mX^{(\ell)}\bigr)
\xrightarrow{\operatorname{SheafConv}_{\theta_\ell}}
\bigl(
G_\ell,
\boldsymbol{\rho}^{(\ell)},
\widetilde{\mX}^{(\ell)}
\bigr)
\xrightarrow{
\operatorname{HiSP}_{\ell}
}
\bigl(G_{\ell+1},\mX^{(\ell+1)}\bigr).
\end{equation}
The subsequent sheaf-convolution block learns a new sheaf on \(G_{\ell+1}\), allowing each hierarchy level to learn diffusion dynamics adapted to its graph resolution.

\subsection{Graph-level architecture}
\label{app:graph_level_architecture}

For graph-level tasks, the model repeatedly applies the hierarchical stage in Equation~\ref{eq:hierarchical_sheaf_stage}. Denoting the composition of the sheaf-convolution and pooling operations at level \(\ell\) by
\[
\mathcal H_\ell
=
\operatorname{HiSP}_\ell
\circ
\operatorname{SheafConv}_{\theta_\ell},
\]
the complete architecture is
\begin{equation}
\label{eq:graph_level_hierarchical_architecture}
\bigl(G_0,\mX^{(0)}\bigr)
\xrightarrow{\mathcal H_0}
\bigl(G_1,\mX^{(1)}\bigr)
\xrightarrow{\mathcal H_1}
\cdots
\xrightarrow{\mathcal H_{L-1}}
\bigl(G_{\cL},\mX^{(\cL)}\bigr)
\xrightarrow{\operatorname{HiSPReadout}}
\mathbf g_G
\xrightarrow{\operatorname{MLP}}
\widehat{\vy}_G .
\end{equation}
The hierarchy progressively replaces groups of nodes with supernodes. The \gls{hisp} readout maps the final coarse representations to a graph descriptor \(\mathbf g_G\), and the prediction head returns one output \(\widehat{\vy}_G\) per graph. No lifting or decoder is required.

\begin{figure}[h!]
    \centering
    \includegraphics[width=\linewidth]{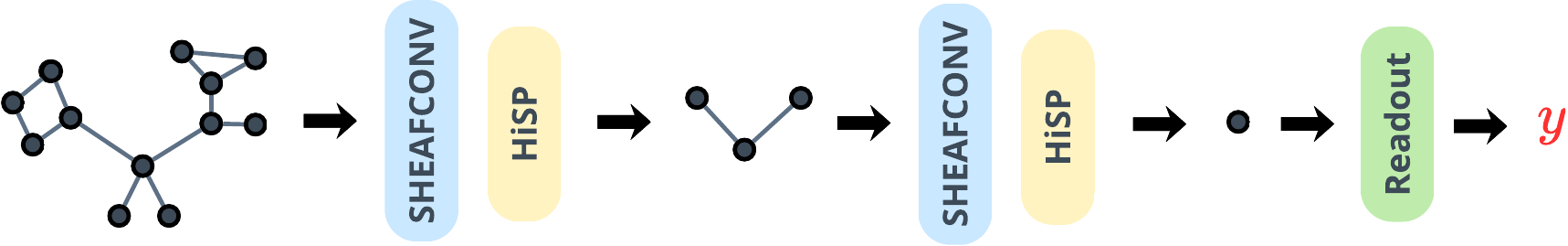}
    \caption{
    Graph-level \gls{hsnn}. Each encoder stage applies a sheaf-convolution block followed by \gls{hisp} pooling. A \gls{hisp} readout
    aggregates the representation at the coarsest resolution to produce the graph-level prediction.
    }
    \label{fig:graph_level_hierarchical_architecture}
\end{figure}

\subsection{Node-level architecture}
\label{app:node_level_architecture}

For node-level tasks, the hierarchical encoder is complemented by a decoder. At encoder level \(\ell\), the output of the sheaf-convolution block is stored
as a skip representation:
\begin{align}
\mK^{(\ell)}
&=
\operatorname{SheafConv}_{\theta_\ell^\downarrow}
\bigl(G_\ell,\mX^{(\ell)}\bigr),
\\
\bigl(G_{\ell+1},\mX^{(\ell+1)}\bigr)
&=
\operatorname{HiSP}_\ell
\left(
G_\ell,\mK^{(\ell)};
\boldsymbol{\rho}^{(\ell)},\pi_\ell
\right).
\end{align}
The restriction maps, cluster assignments, and transport frames computed during pooling are retained for the corresponding decoder stage.

Let \(\pi_\ell:\sV_\ell\to \sV_{\ell+1}\) denote the cluster assignment, and
suppose that \(\pi_\ell(v)=r_a\). The lifted feature at \(v\) is
\begin{equation}
\label{eq:hierarchical_sheaf_lift}
\mU_v^{(\ell)}
=
\tau_{r_a\to v}^{(\ell)}
\mX_{r_a}^{(\ell+1)}.
\end{equation}

The lifted representation is combined with the encoder skip connection and refined by a decoder sheaf-convolution block:
\begin{equation}
\label{eq:hierarchical_sheaf_decoder}
\bar{\mX}^{(\ell)}
=
\operatorname{SheafConv}_{\theta_\ell^\uparrow}
\left(
G_\ell,
\mU^{(\ell)}
\mathbin{\oplus}
\mK^{(\ell)}
\right),
\end{equation}
where \(\oplus\) denotes either addition or projected concatenation. After the last decoder stage, a node-wise prediction head produces
\begin{equation}
\widehat{\vy}_v
=
\operatorname{MLP}
\left(\bar{\mX}_v^{(0)}\right),
\qquad v\in \sV_0.
\end{equation}
Thus, the encoder learns coarse representations, while the decoder restores the original resolution using the same transport geometry employed during pooling.

\begin{figure}[h!]
    \centering
    \includegraphics[width=\linewidth]{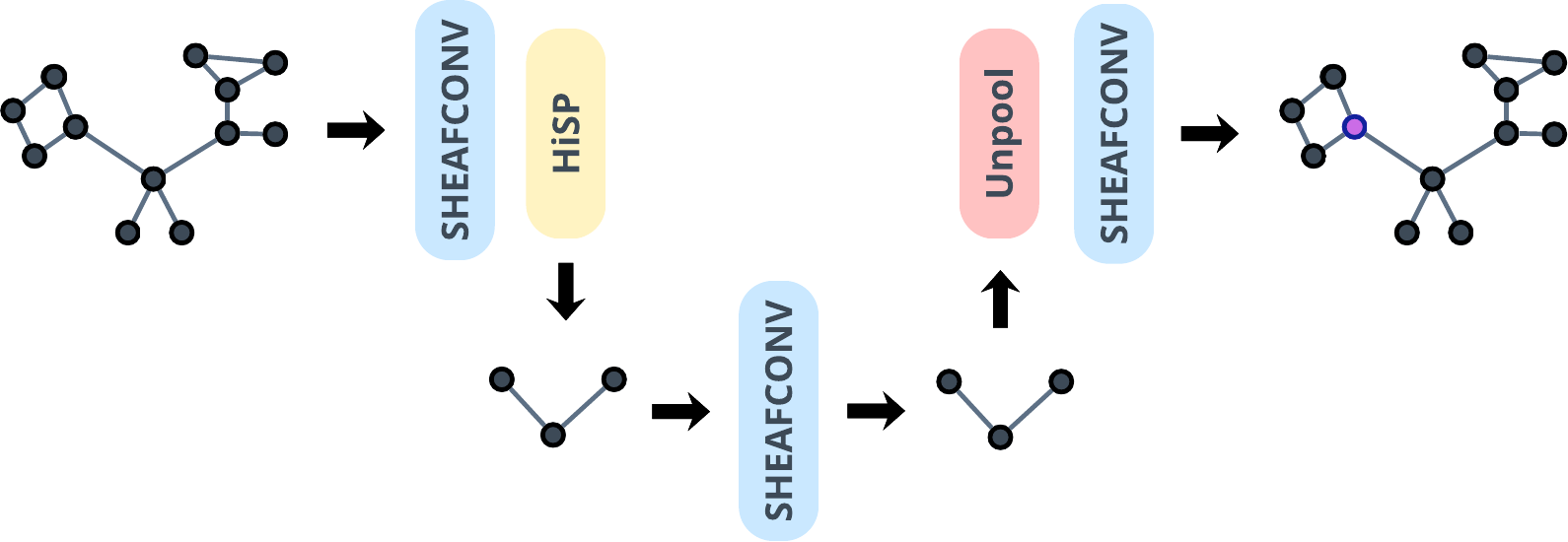}
    \caption{
    Node-level \gls{hsnn}. \gls{hisp} constructs the coarse encoder representation. The decoder lifts it to the original graph, combines it with the corresponding skip connection, and applies a
    sheaf-convolution block before producing node-level predictions.
    }
    \label{fig:node_level_hierarchical_architecture}
\end{figure}

\section{Experimental details}
\label{app:exp_details}

This appendix collects dataset constructions, training protocols, and
extended tables that support Section~\ref{sec:experiments}.
\begin{rebuttalblock}
HierBI, VoteView, Marked PathCopy, and RingTransfer each have their baseline grid in that task's subsection, the grid behind the corresponding table in Section~\ref{sec:experiments}. Appendix~\ref{app:ablations} collects the ablations.
\end{rebuttalblock}
Unless stated otherwise, \gls{hisp} uses Fiedler selection,
\emph{random BFS} tree sampling for transport
(Section~\ref{subsec:reduce}), and weighted transport-mean aggregation.

\subsection{Dataset statistics}
\label{app:dataset_stats}

Table~\ref{tab:dataset_stats_synth}--\ref{tab:dataset_stats_graph}
summarize the datasets used in this paper.
Synthetic sizes are measured from the generators used in our recipes;
geom-gcn and TUDataset sizes match the training loaders.

\begin{table}[h]
\centering
\caption{Synthetic / semi-synthetic benchmarks.}
\label{tab:dataset_stats_synth}
\small
\begin{tabular}{lrrrrr}
\toprule
Dataset & \# Graphs & Avg.\ \# Nodes & Avg.\ \# Edges & \# Feat. & \# Classes \\
\midrule
HierBI & $800$ & $76.0$ & $272.0$ & $4$ & $4$ \\
VoteView & $1012$ & $96.8$ & $866.7$ & $2$ & $2$ \\
RingTransfer & $1200$ & $24.0$ & $24.0$ & $5$ & $5$ \\
Marked PathCopy & $1200$ & $31.5$ & $30.5$ & $8$ & $5$ \\
\bottomrule
\end{tabular}
\end{table}

\begin{table}[h]
\centering
\caption{Heterophilic node-classification datasets (geom-gcn).
Homophily values~$h(G)$ from~\citet{bodnar2022neural}.}
\label{tab:dataset_stats_node}
\small
\begin{tabular}{lrrrrr}
\toprule
Dataset & \# Nodes & \# Edges & \# Feat. & \# Classes & $h(G)$ \\
\midrule
Texas & $183$ & $279$ & $1703$ & $5$ & $0.11$ \\
Wisconsin & $251$ & $450$ & $1703$ & $5$ & $0.21$ \\
Cornell & $183$ & $277$ & $1703$ & $5$ & $0.30$ \\
Chameleon & $2277$ & $18050$ & $2325$ & $5$ & $0.23$ \\
Squirrel & $5201$ & $108536$ & $2089$ & $5$ & $0.22$ \\
\bottomrule
\end{tabular}
\end{table}

\begin{table}[h]
\centering
\caption{TUDataset graph-classification benchmarks~\citep{morris2020tudataset}.}
\label{tab:dataset_stats_graph}
\small
\begin{tabular}{lrrrrr}
\toprule
Dataset & \# Graphs & Avg.\ \# Nodes & Avg.\ \# Edges & \# Feat. & \# Classes \\
\midrule
MUTAG & $188$ & $17.9$ & $19.8$ & $7$ & $2$ \\
PROTEINS & $1113$ & $39.1$ & $72.8$ & $3$ & $2$ \\
ENZYMES & $600$ & $32.6$ & $62.1$ & $3$ & $6$ \\
NCI1 & $4110$ & $29.9$ & $32.3$ & $37$ & $2$ \\
\bottomrule
\end{tabular}
\end{table}

\subsection{HierBI}
\label{app:exp_hierbi}

Four rooms are copies of $K_{8,8}$ ($64$ room nodes). Rooms~$0,1$ form a
meta-left pair and rooms~$2,3$ a meta-right pair. Every left--right pair is
joined by a corridor of length three, so after pooling one supernode per
room the coarse graph is~$K_{2,2}$. Corridor dummies bring each graph to
$76$ nodes and $272$ edges (Figure~\ref{fig:hierbi_chambers}).
Features are four-dimensional, the direct sum of two subspaces.
A hidden sign pattern frames each subspace so that an unsigned room mean
cancels. The label is $y=2b_{\mathrm{coarse}}+b_{\mathrm{fine}}$, the pair
of opposite-versus-crossed bits, and node features include isotropic noise
with $\sigma{=}0.5$. Corridor nodes are zero.

\begin{figure}[h]
    \centering
    \includegraphics[width=0.85\linewidth]{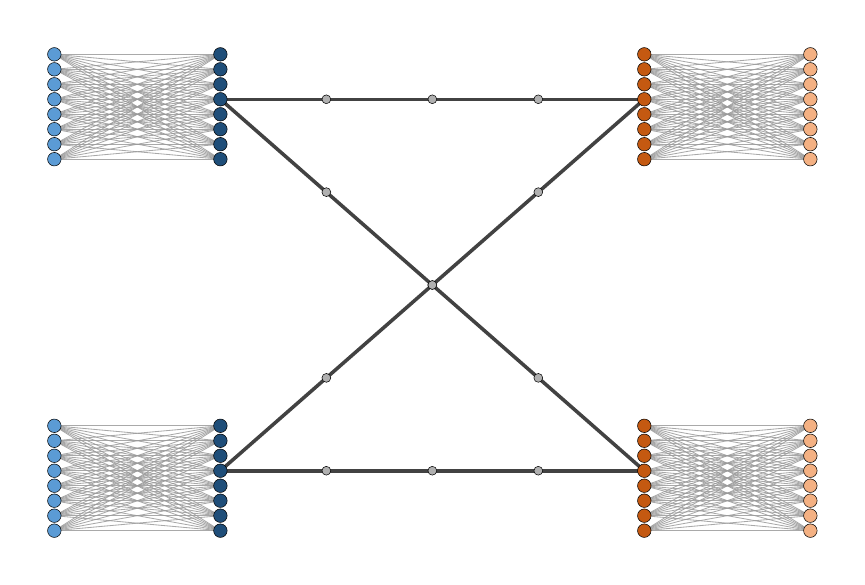}
    \caption{HierBI sample. Colours mark the four rooms; grey nodes are
    corridor dummies that connect every left room to every right room.}
    \label{fig:hierbi_chambers}
\end{figure}

Training uses $400/200/200$ splits, hidden width~$16$, $\cK{=}4$ connected
rooms, and random BFS transport trees, for at most $100$ epochs.
Degree-scored \gls{kmis} uses order \(k{=}2\). Representatives then lie more than two hops apart, so the partition has one cluster in each room and one in each corridor (\(\cK{=}8\)).
HierBI uses a diagonal sheaf with $\cd{=}2$, one pre-pool and two post-pool
sheaf layers, Adam with $\mathrm{lr}{=}0.01$, gradient clipping at $1.0$, and patience~$80$.
\begin{rebuttalblock}
Tables~\ref{tab:hierbi_full_grid} and~\ref{tab:hierbi_gnn_grid} are the full baseline grid behind Table~\ref{tab:independent_sheaves}. Table~\ref{tab:hierbi_full_grid} is \gls{nsd}, with the default aggregate and with the \gls{hisp} readout. Table~\ref{tab:hierbi_gnn_grid} is \gls{gcn} and \gls{gin} with each pooler. Ablations of reduction, transport, readout, root, and cluster size are in Appendix~\ref{app:ablations}.
\end{rebuttalblock}

\begin{rebuttalblock}
\begin{table}[h]
\centering
\caption{HierBI, every run (test acc., \%; $\uparrow$).
The first numeric column is the default aggregate; the second uses the \gls{hisp} readout.
\gls{kmis} uses order $k{=}2$ (one cluster per room and one per corridor, $\cK{=}8$).}
\label{tab:hierbi_full_grid}
\footnotesize
\setlength{\tabcolsep}{4pt}
\begin{tabular}{@{}lcc@{}}
\toprule
Model & default & \gls{hisp} readout \\
\midrule
\gls{nsd} & $43.5{\pm}2.6$ & $47.4{\pm}5.1$ \\
\gls{nsd}, SheafPool & \multicolumn{2}{c}{$47.8{\pm}3.0$} \\
\midrule
\gls{kmis} & $46.1{\pm}4.5$ & $45.6{\pm}3.3$ \\
TopKPool & $34.7{\pm}2.5$ & $43.2{\pm}2.7$ \\
SAGPool & $33.5{\pm}2.3$ & $43.0{\pm}3.0$ \\
MinCutPool & $46.7{\pm}4.4$ & $46.7{\pm}2.5$ \\
DiffPool & $46.6{\pm}2.2$ & $40.3{\pm}7.9$ \\
Eigen & $47.4{\pm}3.8$ & $46.5{\pm}1.7$ \\
\midrule
\gls{nsd}+\gls{hisp} & $89.0{\pm}6.0$ & $\mathbf{93.5{\pm}1.7}$ \\
\bottomrule
\end{tabular}
\end{table}
\end{rebuttalblock}

\begin{rebuttalblock}
\begin{table}[h]
\centering
\caption{HierBI, \gls{gcn} and \gls{gin} (test acc., \%; $\uparrow$). Global mean. \gls{kmis} uses order $k{=}2$ ($\cK{=}8$).}
\label{tab:hierbi_gnn_grid}
\footnotesize
\setlength{\tabcolsep}{4pt}
\begin{tabular}{@{}lcc@{}}
\toprule
Pooler & \gls{gcn} & \gls{gin} \\
\midrule
\gls{kmis} & $40.9{\pm}9.6$ & $28.7{\pm}8.8$ \\
TopKPool & $28.7{\pm}2.4$ & $42.2{\pm}12.0$ \\
SAGPool & $41.2{\pm}4.2$ & $35.8{\pm}5.0$ \\
MinCutPool & $48.0{\pm}3.1$ & $32.0{\pm}10.2$ \\
DiffPool & $48.0{\pm}1.9$ & $46.2{\pm}3.8$ \\
Eigen & $48.4{\pm}4.3$ & $42.8{\pm}2.5$ \\
\bottomrule
\end{tabular}
\end{table}
\end{rebuttalblock}

\subsection{VoteView coalition reading}
\label{app:exp_voteview}

We build graphs from U.S.\ Senate members and roll calls in
Congresses~$90$--$100$ (less-polarized era in which both party-line and
crossed wing coalitions appear in real votes).
Each Congress has a fixed ideology skeleton: senators are placed by
NOMINATE~\citep{poole1985spatial} coordinates, each party is split into
two wings, senators within a wing are linked by ideology $k$NN ($k{=}8$)
plus an MST, and Democrat wings meet Republican wings only through
corridor dummies (Figure~\ref{fig:voteview_hemicycle}).
Each roll call is one instance. Labels are the observed wing-majority
patterns A $(+,+,-,-)$ and B $(+,-,+,-)$ after a global sign flip that
makes the first wing positive; ballots themselves are not features.
Features are relative NOMINATE positions
$\vx_i=\vd_i-\vm$ with respect to the vote-specific cut
$\vm$. Balancing the two classes yields $1012$ graphs
(mean $96.8$ nodes, $866.7$ edges).
The coherence sweep of Table~\ref{tab:voteview_coherence} retains each
senator's feature with probability~$c$ and negates it otherwise.

Training uses a stratified $70/15/15$ split, batch size~$32$, hidden
width~$16$, $\cK{=}4$ connected rooms, random BFS transport trees, a
diagonal sheaf with $\cd{=}1$, one pre-pool and one post-pool layer, Adam
with $\mathrm{lr}{=}0.01$, gradient clipping at $1.0$, at most $100$ epochs, and patience~$30$ on
validation accuracy.

\begin{rebuttalblock}
Table~\ref{tab:voteview_coherence} keeps the best of each family. Tables~\ref{tab:voteview_unpooled_grid}--\ref{tab:voteview_gnn_grid} list every run: unpooled \gls{nsd} under each readout, \gls{nsd} with each pooler under the default aggregate and the \gls{hisp} readout, and \gls{gcn} and \gls{gin} with each pooler.
\end{rebuttalblock}

\begin{rebuttalblock}
\begin{table}[h]
\centering
\caption{VoteView, unpooled \gls{nsd} and \gls{nsd}+\gls{hisp} (test acc., \%; $\uparrow$).}
\label{tab:voteview_unpooled_grid}
\footnotesize
\setlength{\tabcolsep}{3pt}
\begin{tabular}{@{}lcccc@{}}
\toprule
Model & \(c{=}1.0\) & \(c{=}0.7\) & \(c{=}0.4\) & \(c{=}0.1\) \\
\midrule
\gls{nsd} & $83.0{\pm}2.7$ & $80.7{\pm}1.5$ & $83.5{\pm}3.5$ & $80.8{\pm}3.4$ \\
\gls{nsd}, SheafPool & $80.5{\pm}3.7$ & $69.9{\pm}10.5$ & $73.9{\pm}9.4$ & $75.6{\pm}13.6$ \\
\gls{nsd}, \gls{hisp} readout & $85.0{\pm}1.7$ & $84.3{\pm}1.6$ & $85.5{\pm}1.5$ & $82.5{\pm}1.1$ \\
\gls{nsd}+\gls{hisp} & $88.5{\pm}1.9$ & $83.3{\pm}1.6$ & $83.0{\pm}3.1$ & $83.9{\pm}1.5$ \\
\bottomrule
\end{tabular}
\end{table}
\end{rebuttalblock}

\begin{rebuttalblock}
\begin{table}[h]
\centering
\caption{VoteView, \gls{nsd} with a pooler (test acc., \%; $\uparrow$).
Each pooler is shown with the default aggregate and with the \gls{hisp} readout.}
\label{tab:voteview_nsd_pool_grid}
\footnotesize
\setlength{\tabcolsep}{2pt}
\begin{tabular}{@{}llcccc@{}}
\toprule
Pooler & Readout & \(c{=}1.0\) & \(c{=}0.7\) & \(c{=}0.4\) & \(c{=}0.1\) \\
\midrule
\gls{kmis} & global & $83.8{\pm}2.0$ & $83.1{\pm}1.6$ & $83.4{\pm}1.8$ & $83.8{\pm}2.4$ \\
\gls{kmis} & \gls{hisp} & $82.5{\pm}1.9$ & $82.7{\pm}3.2$ & $84.7{\pm}1.3$ & $82.4{\pm}2.6$ \\
TopKPool & global & $77.6{\pm}2.4$ & $74.8{\pm}5.7$ & $75.9{\pm}3.3$ & $74.0{\pm}4.6$ \\
TopKPool & \gls{hisp} & $81.6{\pm}3.8$ & $80.8{\pm}3.3$ & $80.4{\pm}3.1$ & $79.0{\pm}2.9$ \\
SAGPool & global & $80.0{\pm}4.0$ & $75.8{\pm}3.4$ & $76.1{\pm}3.3$ & $74.2{\pm}1.6$ \\
SAGPool & \gls{hisp} & $82.0{\pm}3.0$ & $83.3{\pm}1.3$ & $81.6{\pm}2.4$ & $77.8{\pm}3.6$ \\
MinCutPool & global & $80.1{\pm}2.4$ & $78.8{\pm}2.0$ & $84.3{\pm}2.5$ & $78.7{\pm}1.1$ \\
MinCutPool & \gls{hisp} & $79.3{\pm}3.5$ & $81.7{\pm}2.9$ & $79.5{\pm}1.3$ & $75.9{\pm}4.0$ \\
DiffPool & global & $77.1{\pm}2.8$ & $71.8{\pm}3.9$ & $74.6{\pm}5.3$ & $75.8{\pm}2.5$ \\
DiffPool & \gls{hisp} & $77.5{\pm}1.6$ & $74.2{\pm}2.2$ & $74.8{\pm}7.7$ & $76.5{\pm}2.1$ \\
Eigen & global & $79.6{\pm}3.6$ & $79.1{\pm}1.4$ & $83.7{\pm}3.7$ & $79.5{\pm}2.4$ \\
Eigen & \gls{hisp} & $79.2{\pm}3.3$ & $83.3{\pm}0.7$ & $83.4{\pm}2.9$ & $76.2{\pm}2.4$ \\
\bottomrule
\end{tabular}
\end{table}
\end{rebuttalblock}

\begin{rebuttalblock}
\begin{table}[h]
\centering
\caption{VoteView, \gls{gcn} and \gls{gin} (test acc., \%; $\uparrow$). Global mean.}
\label{tab:voteview_gnn_grid}
\footnotesize
\setlength{\tabcolsep}{2pt}
\begin{tabular}{@{}llcccc@{}}
\toprule
Backbone & Pooler & \(c{=}1.0\) & \(c{=}0.7\) & \(c{=}0.4\) & \(c{=}0.1\) \\
\midrule
\gls{gcn} & \gls{kmis} & $80.7{\pm}3.1$ & $75.8{\pm}3.2$ & $68.9{\pm}1.2$ & $81.8{\pm}1.6$ \\
\gls{gcn} & TopKPool & $76.5{\pm}4.3$ & $61.0{\pm}3.1$ & $70.3{\pm}1.3$ & $68.6{\pm}1.9$ \\
\gls{gcn} & SAGPool & $71.6{\pm}7.1$ & $66.5{\pm}3.2$ & $71.1{\pm}4.4$ & $72.4{\pm}2.4$ \\
\gls{gcn} & MinCutPool & $77.5{\pm}3.0$ & $74.8{\pm}3.5$ & $72.4{\pm}3.4$ & $77.1{\pm}2.0$ \\
\gls{gcn} & DiffPool & $76.7{\pm}2.0$ & $75.7{\pm}2.7$ & $72.7{\pm}2.0$ & $75.6{\pm}2.8$ \\
\gls{gcn} & Eigen & $78.2{\pm}2.0$ & $73.6{\pm}1.8$ & $73.5{\pm}2.0$ & $77.8{\pm}1.7$ \\
\midrule
\gls{gin} & \gls{kmis} & $85.4{\pm}0.9$ & $76.9{\pm}4.3$ & $70.6{\pm}2.0$ & $83.8{\pm}0.8$ \\
\gls{gin} & TopKPool & $76.6{\pm}4.2$ & $61.7{\pm}6.8$ & $68.6{\pm}2.5$ & $74.1{\pm}4.5$ \\
\gls{gin} & SAGPool & $80.1{\pm}2.8$ & $66.5{\pm}6.0$ & $59.5{\pm}4.7$ & $72.2{\pm}4.0$ \\
\gls{gin} & MinCutPool & $82.1{\pm}1.9$ & $76.5{\pm}1.3$ & $76.3{\pm}1.3$ & $82.7{\pm}3.9$ \\
\gls{gin} & DiffPool & $83.8{\pm}1.4$ & $71.0{\pm}4.0$ & $72.0{\pm}1.3$ & $78.7{\pm}3.6$ \\
\gls{gin} & Eigen & $82.9{\pm}1.1$ & $75.0{\pm}3.2$ & $73.5{\pm}3.0$ & $82.5{\pm}1.5$ \\
\bottomrule
\end{tabular}
\end{table}
\end{rebuttalblock}

\subsection{Marked PathCopy}
\label{app:exp_marked_path}

Each graph is a path of length~$24$ with a small random padding
(about~$32$--$35$ nodes). Source and query sit at the path endpoints. Node
features mark the source, query, and distractors, and carry a five-way
one-hot value channel. Wrong-class distractors occupy every second interior
node (Figure~\ref{fig:marked_path_copy}). The model is supervised only at
the query.
We use two Fiedler \gls{hisp} stages at ratio~$0.5$ with random BFS
transport trees, so that after both collapses source and query remain in
distinct coarse rooms. Training uses
$800/200/200$ graphs, $16$ sheaf layers in a U-Net split, hidden width~$64$,
diagonal sheaf with~$\cd{=}2$, Adam with $\mathrm{lr}{=}0.02$, cosine warmup of
$10$ epochs, at most~$150$ epochs, and patience~$50$. Baselines share this
schedule.
\rebuttal{Table~\ref{tab:pathcopy_baselines} is the baseline grid behind Table~\ref{tab:filtering}. The shared-MLP comparison on this partition is the Marked PathCopy column of Table~\ref{tab:transport_pullback}.}
\begin{figure}[h]
    \centering
    \includegraphics[width=0.9\linewidth]{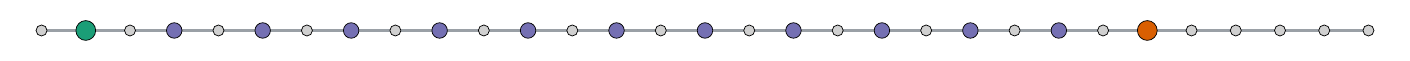}
    \caption{Marked PathCopy sample. Endpoints are the source (left) and
    query (right); every second interior node carries a wrong-class
    distractor.}
    \label{fig:marked_path_copy}
\end{figure}

\begin{rebuttalblock}
\begin{table}[h]
\centering
\caption{Marked PathCopy baselines (test acc., \%; $\uparrow$).
Grid behind Table~\ref{tab:filtering}.
MinCutPool and DiffPool are omitted.
\gls{nsd} with a pooler uses the \gls{hisp} readout; the other rows use a global mean.}
\label{tab:pathcopy_baselines}
\footnotesize
\setlength{\tabcolsep}{4pt}
\begin{tabular}{@{}lc@{}}
\toprule
Model & Accuracy \\
\midrule
\gls{gcn} + \gls{kmis} & $33.5{\pm}7.4$ \\
\gls{gcn} + TopKPool & $30.9{\pm}5.2$ \\
\gls{gcn} + SAGPool & $30.0{\pm}3.6$ \\
\gls{gcn} + Eigen & $43.8{\pm}4.7$ \\
\midrule
\gls{gin} + \gls{kmis} & $20.0{\pm}1.2$ \\
\gls{gin} + TopKPool & $20.2{\pm}1.5$ \\
\gls{gin} + SAGPool & $21.2{\pm}2.6$ \\
\gls{gin} + Eigen & $23.7{\pm}4.4$ \\
\midrule
\gls{nsd} (unpooled) & $55.5{\pm}4.1$ \\
\gls{nsd} + \gls{kmis} & $44.8{\pm}7.0$ \\
\gls{nsd} + TopKPool & $37.0{\pm}4.8$ \\
\gls{nsd} + SAGPool & $30.9{\pm}2.3$ \\
\gls{nsd} + Eigen & $45.4{\pm}4.0$ \\
\midrule
\gls{nsd}+\gls{hisp} & $\mathbf{87.5{\pm}4.6}$ \\
\bottomrule
\end{tabular}
\end{table}
\end{rebuttalblock}

\subsection{RingTransfer}
\label{app:exp_ringtransfer}

We follow~\citet{bodnar2021cwn}: chordless cycles of~$24$ nodes, five classes,
query at node~$0$ and source at the antipode (distance~$12$). Features are
constant ones except for a wiped query and a one-hot source. Models are
masked node classifiers. \gls{nsd}+\gls{hisp} uses one connected stage with
$\cK{=}6$ rooms, random BFS transport trees, diagonal sheaf with~$\cd{=}3$, and
hidden width~$64$.
Train/val/test sizes are $800/200/200$.
Because the diameter equals half the cycle length, any fixed-depth message
passer whose receptive field is shorter than~$12$ hops is information-theoretically
unable to solve the task; coarsening is what reduces that diameter.

\begin{rebuttalblock}
\begin{table}[h]
\centering
\caption{RingTransfer baselines (test acc., \%; $\uparrow$).
Grid behind Table~\ref{tab:ringtransfer}.
Pooler rows use $\cL{=}3$.
MinCutPool and DiffPool are omitted.
\gls{nsd} with a pooler uses the \gls{hisp} readout; the other pooler rows, and unpooled \gls{nsd}, use a global mean.}
\label{tab:ring_baselines}
\footnotesize
\setlength{\tabcolsep}{4pt}
\begin{tabular}{@{}lc@{}}
\toprule
Model & Accuracy \\
\midrule
\gls{gcn} + \gls{kmis} & $20.0{\pm}0.0$ \\
\gls{gcn} + TopKPool & $20.0{\pm}0.0$ \\
\gls{gcn} + SAGPool & $20.0{\pm}0.0$ \\
\gls{gcn} + Eigen & $20.0{\pm}0.0$ \\
\midrule
\gls{gin} + \gls{kmis} & $20.0{\pm}0.0$ \\
\gls{gin} + TopKPool & $20.0{\pm}0.0$ \\
\gls{gin} + SAGPool & $20.0{\pm}0.0$ \\
\gls{gin} + Eigen & $20.0{\pm}0.0$ \\
\midrule
\gls{nsd} + \gls{kmis} & $20.0{\pm}0.0$ \\
\gls{nsd} + TopKPool & $20.8{\pm}1.6$ \\
\gls{nsd} + SAGPool & $20.0{\pm}0.0$ \\
\gls{nsd} + Eigen & $84.0{\pm}32.0$ \\
\midrule
\gls{nsd} ($\cL{=}3$) & $20.0{\pm}0.0$ \\
\gls{nsd} ($\cL{=}6$) & $20.0{\pm}0.0$ \\
\gls{nsd} ($\cL{=}12$) & $72.0{\pm}34.9$ \\
\gls{nsd}+\gls{hisp} ($\cL{=}3$) & $\mathbf{100.0{\pm}0.0}$ \\
\bottomrule
\end{tabular}
\end{table}
\end{rebuttalblock}

\subsection{Ablations}
\label{app:ablations}

\begin{rebuttalblock}
This section collects the comparisons that change one component.
The baseline grids stay in the task subsections above.
\end{rebuttalblock}

\paragraph{Reduction and readout.}
\begin{rebuttalblock}
Table~\ref{tab:hierbi_ablation} holds the HierBI connected partition fixed and crosses the reduction with the readout.
Without transport, every readout stays below $50\%$.
A shared nonlinear MLP on the same connected partition stays at $48.0{\pm}3.3$ (Table~\ref{tab:transport_pullback}).
With \gls{hisp} transport, the \gls{hisp} readout reaches $93.5{\pm}1.7$, a global mean is $89.0{\pm}6.0$ and less stable, and SheafPool falls back to $48.3{\pm}4.0$.
Table~\ref{tab:tree_ablation} keeps the hub and changes only the spanning tree: the three samplings stay within about one point.
Trained checkpoints evaluated on $10$ random spanning trees per test graph move by a mean standard deviation of $0.35\%$.
\end{rebuttalblock}

\begin{table}[t]
\centering
\begin{minipage}[t]{0.55\linewidth}
\centering
\captionof{table}{HierBI \texttt{RED}$\times$\texttt{READOUT} (test acc., \%; $\uparrow$).}
\label{tab:hierbi_ablation}
\small
\begin{tabular}{@{}lccc@{}}
\toprule
Reduce $\backslash$ Readout & \gls{hisp} & mean & SheafPool \\
\midrule
no pooling
  & $47.4{\pm}5.1$
  & $43.5{\pm}2.6$
  & $47.8{\pm}3.0$ \\
mean
  & $46.2{\pm}2.1$
  & $43.4{\pm}4.0$
  & $47.9{\pm}3.2$ \\
\gls{hisp}
  & \rebuttal{$\mathbf{93.5{\pm}1.7}$}
  & \rebuttal{$89.0{\pm}6.0$}
  & \rebuttal{$48.3{\pm}4.0$} \\
\bottomrule
\end{tabular}
\end{minipage}\hfill
\begin{minipage}[t]{0.42\linewidth}
\centering
\begin{rebuttalblock}\captionof{table}{HierBI tree sampling (test acc., \%; $\uparrow$).}
\label{tab:tree_ablation}
\small
\begin{tabular}{@{}lc@{}}
\toprule
Tree Sampling & Accuracy \\
\midrule
Random BFS (default, $M{=}1$) & $93.5{\pm}1.7$ \\
Deterministic BFS             & $92.7{\pm}2.5$ \\
Tree ensemble ($M{=}5$)       & $\mathbf{94.2{\pm}1.8}$ \\
\bottomrule
\end{tabular}\end{rebuttalblock}
\end{minipage}
\end{table}
\paragraph{Transport.}
\begin{rebuttalblock}
Table~\ref{tab:transport_pullback} holds the partition and the decoder fixed and changes the alignment.
The shared MLP is a two-layer residual network with ELU activation function, applied to every node before pooling, with the last layer initialized at zero so that, at the start of the training, pooling corresponds to a mean aggregation.
On HierBI it stays at $48.0{\pm}3.3$, next to the direct mean in Table~\ref{tab:hierbi_ablation} ($46.2{\pm}2.1$), while adjoint transport reaches $93.5{\pm}1.7$.
On Marked PathCopy the shared MLP is $65.5{\pm}11.7$ and adjoint transport is $87.5{\pm}4.6$.
On VoteView at \(c{=}1.0\) the shared MLP ($87.6{\pm}1.4$) and adjoint transport ($88.5{\pm}1.9$) are close.
The damped pseudoinverse still uses the restriction maps and stays with adjoint transport on HierBI ($92.4{\pm}2.0$).
\end{rebuttalblock}

\begin{table}[h]
\centering
\begin{rebuttalblock}
\caption{Shared MLP against sheaf transport (test acc., \%; $\uparrow$).
The partition and the decoder are held fixed.
The shared MLP is one residual map on every node, initialized at coordinate-wise pooling.
VoteView is \(c{=}1.0\).
Adjoint transport is the default.}
\label{tab:transport_pullback}
\small
\begin{tabular}{@{}lccc@{}}
\toprule
Alignment & HierBI & Marked PathCopy & VoteView ($c{=}1$) \\
\midrule
Shared MLP & $48.0{\pm}3.3$ & $65.5{\pm}11.7$ & $87.6{\pm}1.4$ \\
Damped pseudoinverse ($\tau^{+}$) & $92.4{\pm}2.0$ & $82.7{\pm}4.8$ & $86.4{\pm}2.4$ \\
Adjoint ($\tau^{*}$, default) & $\mathbf{93.5{\pm}1.7}$ & $\mathbf{87.5{\pm}4.6}$ & $\mathbf{88.5{\pm}1.9}$ \\
\bottomrule
\end{tabular}
\end{rebuttalblock}
\end{table}
\paragraph{Root and cluster size.}
\begin{rebuttalblock}
Table~\ref{tab:root_ablation} compares root rules on HierBI.
The maximum-degree root is a room hub, tree depth at most $3$, at $93.5{\pm}1.7$.
A minimum-degree root is a degree-$2$ corridor dummy, depth $5$, at $78.2{\pm}19.0$.
A random root falls into opposite bipartite partitions across rooms and drops HierBI to the GNN level, $49.9{\pm}2.2$.
Table~\ref{tab:hierbi_k_robustness} varies the selector and the feature noise.
Automatic \gls{kmis} does not prescribe $\cK$. At $\kappa{=}3$, which separates rooms and corridors into $\cK{=}8$ clusters, it matches the default connected partition ($93.5{\pm}2.9$).
At $\kappa{=}2$ the same count is $80.4{\pm}15.7$, still well above Euclidean \gls{kmis} at that order ($45.6\%$, Table~\ref{tab:hierbi_full_grid}).
Higher feature noise at $\cK{=}4$ stays at $93.0{\pm}2.2$ and $90.7{\pm}2.7$.
\end{rebuttalblock}

\begin{table}[h]
\centering
\begin{minipage}[t]{0.45\linewidth}
\centering
\begin{rebuttalblock}\captionof{table}{Root selection on HierBI (test acc., \%; $\uparrow$).}
\label{tab:root_ablation}
\small
\begin{tabular}{@{}lc@{}}
\toprule
Root Selection Rule & Accuracy \\
\midrule
Max-degree root (default) & $\mathbf{93.5{\pm}1.7}$ \\
Min-degree root           & $78.2{\pm}19.0$ \\
Random root per cluster   & $49.9{\pm}2.2$ \\
\bottomrule
\end{tabular}\end{rebuttalblock}
\end{minipage}\hfill
\begin{minipage}[t]{0.52\linewidth}
\centering
\begin{rebuttalblock}\captionof{table}{Selection and noise robustness on HierBI (test acc., \%; $\uparrow$).}
\label{tab:hierbi_k_robustness}
\small
\begin{tabular}{@{}lc@{}}
\toprule
Selection / Noise Setting & Accuracy \\
\midrule
Connected ($\cK{=}4$, default, $\sigma{=}0.5$) & $\mathbf{93.5{\pm}1.7}$ \\
\gls{kmis} ($\kappa{=}3$, auto $\cK{=}8$)      & $\mathbf{93.5{\pm}2.9}$ \\
\gls{kmis} ($\kappa{=}2$, auto $\cK{=}8$)      & $80.4{\pm}15.7$ \\
Connected ($\cK{=}4$, noise $\sigma{=}0.75$)   & $93.0{\pm}2.2$ \\
Connected ($\cK{=}4$, noise $\sigma{=}1.00$)   & $90.7{\pm}2.7$ \\
\bottomrule
\end{tabular}\end{rebuttalblock}
\end{minipage}
\end{table}
\paragraph{Readout descriptors.}
\begin{rebuttalblock}
Table~\ref{tab:readout_descriptors} splits the \gls{hisp} readout into its three descriptors.
Gram compares channels inside a stalk, so it stays meaningful when supernodes do not share a basis.
On HierBI, Gram alone is $93.3{\pm}2.7$, and the three descriptors together are the $93.5{\pm}1.7$ of Table~\ref{tab:independent_sheaves}.
On VoteView at $c{=}1.0$, the three descriptors together are $88.5{\pm}1.9$ and Gram alone is $88.2{\pm}2.3$.
That all-three cell is the \gls{nsd}+\gls{hisp} entry in Table~\ref{tab:voteview_coherence}.
\end{rebuttalblock}

\begin{table}[h]
\centering
\begin{rebuttalblock}\caption{Ablation of the three \gls{hisp} readout descriptors (test acc., \%; $\uparrow$). Random BFS. The HierBI all-three entry is the model in Table~\ref{tab:independent_sheaves}. The VoteView all-three entry is the \gls{nsd}+\gls{hisp} cell of Table~\ref{tab:voteview_coherence}. Bold is the best mean in each row.}
\label{tab:readout_descriptors}
\footnotesize
\setlength{\tabcolsep}{3pt}
\resizebox{\linewidth}{!}{%
\begin{tabular}{@{}lccccccc@{}}
\toprule
Dataset & $\{\mathrm{Gram}\}$ & $\{\mathrm{mean}\}$ & $\{\mathrm{stat}\}$ & $\{\mathrm{Gram,mean}\}$ & $\{\mathrm{Gram,stat}\}$ & $\{\mathrm{mean,stat}\}$ & $\{\mathrm{Gram,mean,stat}\}$ \\
\midrule
HierBI & $93.3{\pm}2.7$ & $74.5{\pm}22.6$ & $88.9{\pm}2.8$ & $93.3{\pm}1.1$ & $83.8{\pm}18.9$ & $92.7{\pm}1.5$ & $\mathbf{93.5{\pm}1.7}$ \\
VoteView ($c{=}1.0$) & $88.2{\pm}2.3$ & $84.3{\pm}1.7$ & $84.8{\pm}1.9$ & $87.2{\pm}1.8$ & $87.3{\pm}0.9$ & $84.8{\pm}2.4$ & $\mathbf{88.5{\pm}1.9}$ \\
\bottomrule
\end{tabular}}\end{rebuttalblock}
\end{table}
\paragraph{Operator-norm clipping.}
\begin{rebuttalblock}
Table~\ref{tab:clipping_ablation} turns the normalized transport of Appendix~\ref{app:non_expansive_transport} on and off.
For orthogonal maps the clip divides by one, so the transport is unchanged and Marked PathCopy matches exactly.
The HierBI bundle means differ by $0.3$ points, inside the seed standard deviations.
Removing the clip costs $8.7$ points on diagonal HierBI and $18.8$ points on the general-sheaf Marked PathCopy U-Net.
\end{rebuttalblock}

\begin{table}[h]
\centering
\begin{rebuttalblock}\caption{Ablation of non-expansive operator-norm clipping ($\widehat{\tau}_{u\to v}$) across sheaf parameterizations (test acc., \%; $\uparrow$).}
\label{tab:clipping_ablation}
\small
\begin{tabular}{@{}llccc@{}}
\toprule
Task & Sheaf Type & With Clipping (Default) & Without Clipping & $\Delta$ \\
\midrule
HierBI          & \texttt{diag} ($\cd{=}2$)    & $\mathbf{93.5{\pm}1.7}$ & $84.8{\pm}18.7$ & $-8.7$ \\
HierBI          & \texttt{general} ($\cd{=}2$) & $91.0{\pm}5.0$          & $91.3{\pm}5.6$  & $+0.3$ \\
HierBI          & \texttt{bundle} $O(\cd)$     & $91.6{\pm}2.5$          & $91.9{\pm}0.9$  & $+0.3$ \\
Marked PathCopy & \texttt{general} ($\cd{=}2$) & $\mathbf{81.8{\pm}9.3}$ & $63.0{\pm}13.3$ & $-18.8$ \\
Marked PathCopy & \texttt{bundle} $O(\cd)$     & $\mathbf{84.8{\pm}2.3}$ & $\mathbf{84.8{\pm}2.3}$ & $\phantom{-}0.0$ \\
\bottomrule
\end{tabular}\end{rebuttalblock}
\end{table}

\FloatBarrier

\subsection{Node Classification}
\label{app:node_class}

We follow the evaluation and model-selection protocol of
\citet{bodnar2022neural}. All models are evaluated on the same ten fixed
Geom-GCN splits, each containing \(48\%\), \(32\%\), and \(20\%\) of the
nodes from every class for training, validation, and testing, respectively.
For each split and hyperparameter configuration, training uses Adam and the
checkpoint with the highest validation accuracy is retained. Following the
discrete \gls{nsd} protocol, training is limited to \(1000\) epochs with patience
\(100\) on Chameleon and Squirrel, and to \(500\) epochs with patience \(200\)
on the remaining datasets.

We run separate hyperparameter sweeps for the one- and two-stage sheaf U-Net
architectures. Each sweep configuration is evaluated on all ten splits, and
the configuration with the highest mean validation accuracy is selected. We report the mean and
standard deviation of the test accuracies obtained at the corresponding
validation-selected checkpoints. When reporting a single \gls{nsd}+\gls{hisp}
result, we similarly select between the one- and two-stage configurations
using mean validation accuracy. The complete per-stage results are reported
in Table~\ref{tab:node_class_full}.

\begin{table}[h]
\centering
\caption{Node classification (test acc., \%).
NSD uses our reproduction on Texas, Wisconsin, Cornell, Chameleon, and
Squirrel (released recipes); Film, CiteSeer, PubMed, and Cora quote the
best published NSD variant of~\citet{bodnar2022neural}.
HiSP: one- and two-stage sheaf U-Nets.
H2GCN / GCN from~\citet[Table~1]{bodnar2022neural}.}
\label{tab:node_class_full}
\small
\setlength{\tabcolsep}{3pt}
\begin{tabular}{@{}lccccc@{}}
\toprule
Dataset & \gls{nsd} & HiSP (1) & HiSP (2) & H2GCN & \gls{gcn} \\
\midrule
Film
  & $\mathbf{37.81{\pm}1.15}$
  & $37.0{\pm}1.2$
  & $36.9{\pm}1.2$
  & $35.70{\pm}1.00$
  & $27.32{\pm}1.10$ \\
PubMed
  & $89.49{\pm}0.40$
  & $89.5{\pm}0.4$
  & $\mathbf{89.7{\pm}0.3}$
  & $89.49{\pm}0.38$
  & $88.42{\pm}0.50$ \\
CiteSeer
  & $77.14{\pm}1.85$
  & $76.6{\pm}1.4$
  & $\mathbf{77.4{\pm}1.4}$
  & $77.11{\pm}1.57$
  & $76.50{\pm}1.36$ \\
Cora
  & $87.30{\pm}1.15$
  & $87.3{\pm}1.2$
  & $87.2{\pm}1.0$
  & $\mathbf{87.87{\pm}1.20}$
  & $86.98{\pm}1.27$ \\
Squirrel
  & $57.00{\pm}1.90$
  & $\mathbf{58.0{\pm}2.1}$
  & $57.1{\pm}1.8$
  & $36.48{\pm}1.86$
  & $53.43{\pm}2.01$ \\
Chameleon
  & $68.00{\pm}1.22$
  & $69.0{\pm}1.9$
  & $\mathbf{69.3{\pm}1.9}$
  & $60.11{\pm}2.15$
  & $64.82{\pm}2.24$ \\
Texas
  & $\mathbf{85.14{\pm}5.30}$
  & $81.4{\pm}5.1$
  & $85.1{\pm}5.8$
  & $84.86{\pm}7.23$
  & $55.14{\pm}5.16$ \\
Cornell
  & $\mathbf{83.24{\pm}6.37}$
  & $83.0{\pm}5.1$
  & $82.4{\pm}6.8$
  & $82.70{\pm}5.28$
  & $60.54{\pm}5.30$ \\
Wisconsin
  & $\mathbf{89.41{\pm}2.93}$
  & $87.5{\pm}3.5$
  & $87.8{\pm}1.7$
  & $87.65{\pm}4.98$
  & $51.76{\pm}3.06$ \\
\bottomrule
\end{tabular}
\end{table}

\subsection{LRIM protocol and extended grids}
\label{app:exp_lrim}

\paragraph{Protocol.}
We use the official inductive split of
\texttt{lrim\_16\_0.6\_10k}~\citep{mathys2026lrim} ($80/10/10$), metric
$\log_{10}(\mathrm{MSE})$.
\rebuttal{We sweep depth $\cL\in\{4,8,12,16\}$ and compare unpooled sheaf backbones---diagonal \gls{nsd} (\texttt{diag\_d2}), flat vector bundles in BuNN~\citep{bamberger2025bundle} (\texttt{bunn\_d2}), and directed conformal sheaves in CSNN~\citep{ribeiro2025cooperative} (\texttt{csnn\_d2})---against \gls{hisp} with one or two pooling stages.}
For \texttt{diag\_d2}, we additionally sweep Fiedler selection at ratios $\{0.2,0.5,0.8\}$ and \gls{kmis} selection with $\kappa\in\{1,2,3,4\}$.
Layer budgets are matched by splitting sheaf layers across encoder/decoder stages. Each cell is mean~$\pm$~std over five seeds.

\paragraph{Extended results.}
\rebuttal{Table~\ref{tab:lrim_sheaf_backbones} reports the complete one- and two-stage results for BuNN and CSNN at pool ratio $r{=}0.5$, and Table~\ref{tab:lrim_full} reports the full \texttt{diag\_d2} grid over selection strategies and pooling ratios.}

\begin{table}[h]
\centering
\begin{rebuttalblock}\caption{LRIM-16-hard across BuNN and CSNN ($r{=}0.5$, test $\log_{10}\mathrm{MSE}\downarrow$).
Within each sheaf, bold marks the best of the unpooled model and the one- and two-stage \gls{hisp} variants.}
\label{tab:lrim_sheaf_backbones}
\small
\begin{tabular}{@{}lcccc@{}}
\toprule
Model & $\cL{=}4$ & $\cL{=}8$ & $\cL{=}12$ & $\cL{=}16$ \\
\midrule
depth-$\cL$ cutoff & $-0.287$ & $-2.066$ & $-13.782$ & $-13.782$ \\
\midrule
BuNN (unpooled)               & $-1.182{\pm}.045$ & $-1.613{\pm}.071$ & $\mathbf{-1.750{\pm}.061}$ & $-1.814{\pm}.054$ \\
BuNN + \gls{hisp} (1 stage)   & $-1.245{\pm}.086$ & $\mathbf{-1.666{\pm}.038}$ & $-1.721{\pm}.108$ & $\mathbf{-1.869{\pm}.042}$ \\
BuNN + \gls{hisp} (2 stages)  & $\mathbf{-1.322{\pm}.062}$ & $-1.660{\pm}.028$ & $-1.731{\pm}.048$ & $-1.743{\pm}.066$ \\
\midrule
CSNN (unpooled)               & $-1.559{\pm}.019$ & $\mathbf{-1.792{\pm}.022}$ & $\mathbf{-1.871{\pm}.017}$ & $-1.847{\pm}.035$ \\
CSNN + \gls{hisp} (1 stage)   & $-1.543{\pm}.045$ & $-1.782{\pm}.025$ & $-1.866{\pm}.030$ & $-1.899{\pm}.027$ \\
CSNN + \gls{hisp} (2 stages)  & $\mathbf{-1.618{\pm}.015}$ & $-1.773{\pm}.026$ & $-1.823{\pm}.023$ & $\mathbf{-1.910{\pm}.019}$ \\
\bottomrule
\end{tabular}\end{rebuttalblock}
\end{table}

\begin{table}[h]
\centering
\caption{LRIM-16-hard, \texttt{diag\_d2} grid (test $\log_{10}$MSE $\downarrow$).}
\label{tab:lrim_full}
\small
\setlength{\tabcolsep}{4pt}
\begin{tabular}{@{}lcccc@{}}
\toprule
Model & $\cL{=}4$ & $\cL{=}8$ & $\cL{=}12$ & $\cL{=}16$ \\
\midrule
depth-$\cL$ cutoff
  & $-0.287$ & $-2.066$ & $-13.782$ & $-13.782$ \\
\midrule
NSD
  & $-0.682{\pm}.026$
  & $-0.793{\pm}.278$
  & $-0.607{\pm}.135$
  & $-0.804{\pm}.185$ \\
\midrule
\multicolumn{5}{l}{\emph{One stage, connected}} \\
$r{=}0.2$
  & $-0.964{\pm}.035$
  & $-1.358{\pm}.052$
  & $-1.289{\pm}.224$
  & $-0.688{\pm}.252$ \\
$r{=}0.5$
  & $-0.787{\pm}.043$
  & $-1.250{\pm}.082$
  & $-1.267{\pm}.283$
  & $-0.973{\pm}.333$ \\
$r{=}0.8$
  & $-0.684{\pm}.010$
  & $-1.016{\pm}.360$
  & $-1.022{\pm}.238$
  & $-0.911{\pm}.304$ \\
\midrule
\multicolumn{5}{l}{\emph{One stage, \gls{kmis}}} \\
$\kappa{=}1$
  & $-0.823{\pm}.042$
  & $-1.270{\pm}.114$
  & $-1.170{\pm}.226$
  & $-1.036{\pm}.245$ \\
$\kappa{=}2$
  & $-1.023{\pm}.041$
  & $-1.419{\pm}.096$
  & $-1.290{\pm}.276$
  & $-0.916{\pm}.392$ \\
$\kappa{=}3$
  & $-1.069{\pm}.047$
  & $-1.372{\pm}.074$
  & $-1.197{\pm}.344$
  & $-0.741{\pm}.182$ \\
$\kappa{=}4$
  & $-1.058{\pm}.065$
  & $-1.350{\pm}.076$
  & $-0.991{\pm}.420$
  & $-0.443{\pm}.393$ \\
\midrule
\multicolumn{5}{l}{\emph{Two stages, connected}} \\
$r{=}0.2$
  & $-0.899{\pm}.085$
  & $-1.356{\pm}.043$
  & $-1.389{\pm}.047$
  & $-1.465{\pm}.106$ \\
$r{=}0.5$
  & $-0.787{\pm}.041$
  & $-1.299{\pm}.093$
  & $-1.433{\pm}.062$
  & $-1.500{\pm}.060$ \\
$r{=}0.8$
  & $-0.673{\pm}.010$
  & $-1.301{\pm}.052$
  & $-1.285{\pm}.105$
  & $-1.376{\pm}.226$ \\
\midrule
\multicolumn{5}{l}{\emph{Two stages, \gls{kmis}}} \\
$\kappa{=}1$
  & $-0.893{\pm}.029$
  & $-1.314{\pm}.104$
  & $-1.491{\pm}.023$
  & $-1.441{\pm}.139$ \\
$\kappa{=}2$
  & $-0.997{\pm}.055$
  & $-1.425{\pm}.049$
  & $-1.438{\pm}.100$
  & $-1.466{\pm}.152$ \\
$\kappa{=}3$
  & $-0.946{\pm}.099$
  & $-1.429{\pm}.022$
  & $-1.409{\pm}.046$
  & $-1.454{\pm}.086$ \\
$\kappa{=}4$
  & $-0.875{\pm}.063$
  & $-1.372{\pm}.060$
  & $-1.368{\pm}.086$
  & $-1.506{\pm}.063$ \\
\bottomrule
\end{tabular}
\end{table}

\subsection{Graph-level sweep protocol}
\label{app:exp_graph}

We run Bayesian sweeps on MUTAG, PROTEINS, ENZYMES, and NCI1.
Each trial trains a $10$-fold stratified split and logs mean validation loss.
Search includes sheaf type/dimension (diagonal, bundle, general; with or
without high/low-pass residuals), depth~$\cL\in[2,8]$, hidden width in
$\{8,16,32,64\}$, dropout, input dropout, and sheaf/weight decay.
Pooled variants additionally sweep the number of \gls{hisp} stages
($\{1,2\}$) with \gls{kmis} ($k{=}1$), random BFS trees, and
weighted aggregation.

\section{NSD port and Sheaf U-Net}
\label{app:sheaf_unet}

\subsection{NSD reproduction and HiSP comparison}
\label{app:bodnar_reproduce}

On the five geom-gcn datasets with public recipes
(Texas, Wisconsin, Cornell, Chameleon, Squirrel), we reproduce
unpooled \gls{nsd} from the released checkpoints of~\citet{bodnar2022neural}
(diagonal or $O(\cd)$-bundle sheaves; seed~$43$).
Our PyG/tgp port matches the original logits on all fifty folds
(max $\lvert\Delta\mathrm{logit}\rvert\approx 3{\times}10^{-5}$).
For Film, CiteSeer, PubMed, and Cora the recipes were not released, so the
NSD column quotes the best published variant from their Table~1.
Table~\ref{tab:node_class_full} reports that backbone against one- and
two-stage \gls{hisp} sheaf U-Nets under the same Geom-GCN splits and
validation protocol (Appendix~\ref{app:node_class}), with H2GCN and GCN
included for context.

\subsection{Sheaf U-Net details}
\label{app:sheaf_unet_results}

Node-level \gls{hsnn}s embed the hierarchical encoder in a sheaf U-Net with
transport-based lifting (Appendix~\ref{app:node_level_architecture}).
Checkpoints are selected by validation accuracy over ten folds; the
corresponding test accuracies are Table~\ref{tab:node_class_full}.

\subsection{SheafConv port timing}
\label{app:forward_timing}

To check that the PyG/tgp sheaf-convolution port is faithful, we time
eval forwards on several node datasets and a MUTAG batch (CUDA, RTX~3080)
with and without a cached Laplacian / pooling pretransform
(Figure~\ref{fig:forward_timing}).
NSD (original) and NSD (PyG) match under both settings; caching the
pretransform yields about a $4\times$ speedup. The same pattern holds for a
one-stage sheaf U-Net and for \gls{nsd}+\gls{hisp} on MUTAG.

\begin{figure}[h]
    \centering
    \includegraphics[width=\linewidth]{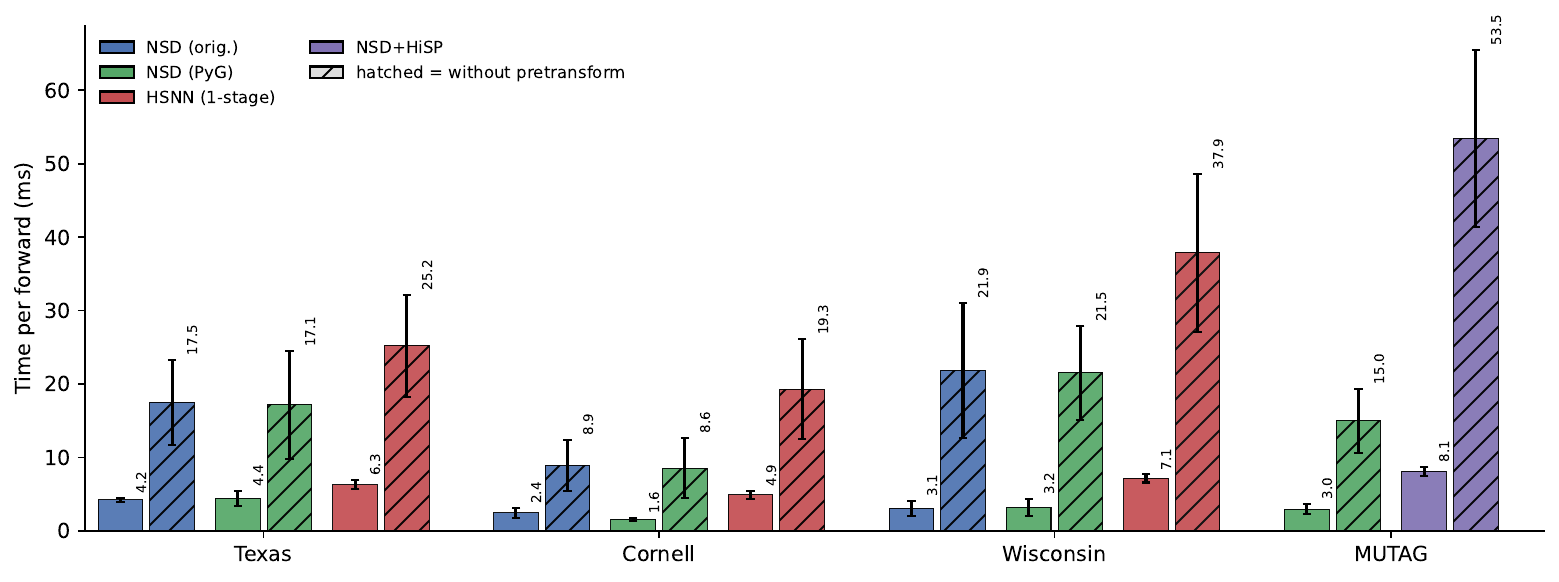}
    \caption{Eval forward time with vs.\ without sheaf/pool pretransform
    on Texas, Cornell, Wisconsin, and MUTAG. NSD (original) and NSD (PyG)
    match; pretransform caching removes a large constant cost.}
    \label{fig:forward_timing}
\end{figure}

\begin{rebuttalblock}Table~\ref{tab:runtime_memory} complements Figure~\ref{fig:forward_timing} with training-step runtime (forward, backward, and optimizer step) and peak allocated GPU memory across \gls{gcn}, \gls{gin}, \gls{gcn}+Eigen, unpooled \gls{nsd}, and \gls{nsd}+\gls{hisp} (RTX~3080).
With \texttt{SheafNetPretransform}, adding \gls{hisp} to \gls{nsd} increases training-step time by $1.23\times$ on Chameleon and $1.34\times$ on Texas while reducing peak training GPU memory by $9.2\%$ on Chameleon ($902.7\text{ MB}$ vs.\ $994.6\text{ MB}$) and $19.0\%$ on HierBI ($83.4\text{ MB}$ vs.\ $103.0\text{ MB}$), because post-pooling sheaf diffusion layers operate on the smaller coarse graph $G_1$.\end{rebuttalblock}

\begin{table}[h]
\centering
\begin{rebuttalblock}\caption{Runtime (ms) and peak GPU memory (MB) on an NVIDIA RTX 3080.}
\label{tab:runtime_memory}
\footnotesize
\setlength{\tabcolsep}{2pt}
\begin{tabular}{@{}llcccc@{}}
\toprule
Dataset & Model & Eval Fwd (ms) & Train Step (ms) & Mem Eval (MB) & Mem Train (MB) \\
\midrule
Chameleon & \gls{gcn} / \gls{gin}                         & $4.75$ / $2.68$ & $4.41$ / $5.09$ & $203.9$ / $534.3$ & $209.4$ / $535.0$ \\
($N{=}2277$) & \gls{gcn}+Eigen                  & $2.85{\pm}0.26$ & $4.99{\pm}0.35$ & $205.8$           & $210.0$ \\
          & \gls{nsd} (unpooled)                    & $7.05{\pm}0.31$ & $17.17{\pm}0.18$ & $374.5$          & $994.6$ \\
          & \textbf{\gls{nsd}+\gls{hisp} (PT)} & $\mathbf{10.84{\pm}0.67}$ & $\mathbf{21.15{\pm}1.05}$ & $\mathbf{397.4}$ & $\mathbf{902.7}$ \\
          & \gls{nsd}+\gls{hisp} (no PT) & $441.01{\pm}19.06$ & $461.85{\pm}38.37$ & $432.5$       & $931.9$ \\
\midrule
Texas     & \gls{gcn}+Eigen                     & $2.55{\pm}0.46$ & $4.32{\pm}0.14$ & $73.6$            & $74.0$ \\
($N{=}183$) & \gls{nsd} (unpooled)                  & $6.36{\pm}0.16$ & $19.16{\pm}1.67$ & $74.5$           & $80.3$ \\
          & \textbf{\gls{nsd}+\gls{hisp} (PT)} & $\mathbf{10.51{\pm}0.29}$ & $\mathbf{25.63{\pm}1.33}$ & $\mathbf{76.1}$ & $\mathbf{81.2}$ \\
\midrule
MUTAG     & \gls{gcn}+Eigen                     & $2.60{\pm}0.86$ & $4.26{\pm}0.98$ & $65.1$            & $65.5$ \\
(batch $32$) & \gls{nsd} (unpooled)                 & $4.10{\pm}0.26$ & $9.67{\pm}0.45$ & $66.8$            & $71.5$ \\
          & \textbf{\gls{nsd}+\gls{hisp} (PT)} & $\mathbf{10.28{\pm}0.33}$ & $\mathbf{18.16{\pm}0.50}$ & $\mathbf{66.4}$ & $\mathbf{70.6}$ \\
\midrule
HierBI    & \gls{gcn}+Eigen                     & $2.32{\pm}0.10$ & $4.03{\pm}0.33$ & $67.6$            & $67.6$ \\
(batch $32$) & \gls{nsd} (unpooled)                 & $4.54{\pm}0.62$ & $9.93{\pm}0.24$ & $77.0$            & $103.0$ \\
          & \textbf{\gls{nsd}+\gls{hisp} (PT)} & $\mathbf{12.84{\pm}2.98}$ & $\mathbf{21.90{\pm}1.70}$ & $\mathbf{76.4}$ & $\mathbf{83.4}$ \\
\bottomrule
\end{tabular}\end{rebuttalblock}
\end{table}

\FloatBarrier
\section{Cached sheaf indices and HiSP trees}
\label{app:sheaf_cache}

\gls{tgp} stores a precomputed pooling assignment on each graph and batches those assignments with its own collation.
A sheaf network needs two further pieces, both fixed once the clusters and the graph topology are fixed: the index structure of the sheaf Laplacian, and the rooted tree used by \gls{hisp}.
\texttt{SheafNetPretransform} writes them into two lists on the same \texttt{Data} object.

\begin{lstlisting}[language={},escapeinside={(*}{*)},float=false,caption={Cached sheaf / HiSP pretransform.},label={lst:sheaf_pretransform}]
SheafNetPretransform(G):
    sheaf_data[0] <- LaplacianIndices(G)     # level 0
    for level (*$\ell$*) = 0, ..., (*$\cL-1$*):
        pooled_data[(*$\ell$*)] <- HiSPSelect((*$G_\ell$*))    # clusters, roots, parents
        (*$G_{\ell+1}$*)        <- coarse graph of pooled_data[(*$\ell$*)]
        sheaf_data[(*$\ell+1$*)] <- LaplacianIndices((*$G_{\ell+1}$*))
\end{lstlisting}

\noindent
\texttt{LaplacianIndices} is a \texttt{SheafData} object. Its fields are integer indices into edges, into nodes, or into the lifted stalks of size \(|\sV|\cdot\cd\). \texttt{HiSPSelect} stores, for each cluster, the assignment and the parent of every node in the transport tree.
Parents are kept as a difference from the node index, so concatenation does not require a second offset rule. Batching uses those increments explicitly:

\begin{lstlisting}[language={},escapeinside={(*}{*)},float=false,caption={Batch index increments for sheaf fields.},label={lst:sheaf_inc}]
SheafData.__inc__(key):
    edge index   -> (*$|\sE|$*)
    node index   -> (*$|\sV|$*)
    lifted index -> (*$|\sV|\cdot\cd$*)
    tree parent  -> 0          # already stored as parent - node
\end{lstlisting}

\noindent
\texttt{SheafNetBatch} collates ordinary PyG fields, the \texttt{pooled\_data} list, and the \texttt{sheaf\_data} list in one step.
At training time the sheaf layer reads \texttt{sheaf\_data[$\ell$]} and \gls{hisp} reads the tree in \texttt{pooled\_data[$\ell$]}, without rebuilding either.

\end{document}